\documentclass[twoside,11pt]{article}
\usepackage[preprint]{jmlr2e}
\usepackage{amsmath}
\usepackage{amssymb}
\usepackage{array}
\usepackage{cite}
\usepackage[shortlabels]{enumitem}
\usepackage{graphicx}
\usepackage{url}
\usepackage{color}
\usepackage{mathtools}
\usepackage{tikz}
\usetikzlibrary{shapes.misc,decorations.pathreplacing,decorations.markings,patterns,arrows.meta}

\newcommand{\NN}{{\mathbb N}}
\newcommand{\RR}{{\mathbb R}}

\newcommand{\cM}{{\mathcal{M}}}

\newcommand{\Prob}{{\mathbb{P}}}

\newcommand{\cP}{{\mathcal{P}}}

\newcommand{\argmax}{\mathop{\mathrm{argmax}}}  
\newcommand{\argmin}{\mathop{\mathrm{argmin}}}

\usepackage{lastpage}
\jmlrheading{22}{2021}{1-\pageref{LastPage}}{1/20; Revised 2/21}{2/21}{20-006}{Melkior Ornik, Ufuk Topcu} 
\ShortHeadings{Learning and Planning for TVMDPs Using MLE}{Ornik and Topcu}

\newtheorem{problem}[theorem]{Problem}

\firstpageno{1}

\begin{document}

\title{Learning and Planning for Time-Varying MDPs \\ Using Maximum Likelihood Estimation}

\author{\name Melkior Ornik \email mornik@illinois.edu \\
       \addr Department of Aerospace Engineering and the Coordinated Science Laboratory\\
       University of Illinois at Urbana-Champaign\\
       Urbana, IL 61801, USA
       \AND
       \name Ufuk Topcu \email utopcu@utexas.edu \\
       \addr Dept. of Aero. Eng. and Eng. Mechanics and the Oden Inst. for Computational Eng. and Sciences\\
       University of Texas at Austin\\
       Austin, TX 78712, USA}

\editor{Amos Storkey}

\maketitle

\begin{abstract}%
This paper proposes a formal approach to online learning and planning for agents operating in a priori unknown, time-varying environments. The proposed method computes the maximally likely model of the environment, given the observations about the environment made by an agent earlier in the system run and assuming knowledge of a bound on the maximal rate of change of system dynamics. Such an approach generalizes the estimation method commonly used in learning algorithms for unknown Markov decision processes with time-invariant transition probabilities, but is also able to quickly and correctly identify the system dynamics following a change. Based on the proposed method, we generalize the exploration bonuses used in learning for time-invariant Markov decision processes by introducing a notion of uncertainty in a learned time-varying model, and develop a control policy for time-varying Markov decision processes based on the exploitation and exploration trade-off. We demonstrate the proposed methods on four numerical examples: a patrolling task with a change in system dynamics, a two-state MDP with periodically changing outcomes of actions, a wind flow estimation task, and a multi-armed bandit problem with periodically changing probabilities of different rewards.
\end{abstract}


\section{Introduction}
A variety of intelligent agents---notably, autonomous systems---are commonly required to operate in unknown environments \citep{Elf90,Gaoetal14,Heretal15}, necessitating the use of learning in order to complete their tasks. While methods for learning and planning for agents in unknown environments exist in a variety of frameworks \citep{Foretal95,Sanetal97,KeaSin02,AlTetal07,SutBar18}, most of them assume that the environment in which the agent operates is unchanged over the course of the agent's operation. Such an assumption allows for construction of an estimate of the dynamics by performing repeated experiments and observing the outcomes \citep{SutBar18,Sutetal99, KeaSin02,KolNg09,StrLit08,FuTop14,HauSto15,Ornetal18}. However, it is often not realistic for systems operating on long-term missions outside a strictly controlled environment. Taking an example of an extraterrestrial rover mission, changes in the environment may be a consequence of regular, predictable events such as intra-day or seasonal temperature variations \citep{Vasetal17} or may result from more complex phenomena that are difficult to predict, e.g., terrain changes due to wind---see the work of \citet{Zim14} for a detailed study. 

In contrast to assuming time-invariance, accounting for time-varying changes in the environment presents a major challenge to learning and planning. A naive approach---restarting the learning process whenever the environment changes---does not make sense: the environment is possibly continually changing. Restarting the learning process whenever the environment \emph{sufficiently} changes, or sufficient length of time passes, would both neglect the environmental changes between the process restarts and rely on heuristics in deciding when to restart learning. Restarting too often will lead to the agent spending too much time on learning, and lacking time to perform its task. On the other hand, restarting the learning process too rarely can lead to unreliable learning outcomes. A sliding window approach \citep{Huaetal10, Gajetal18}, where only information gathered within a fixed period of time prior to the time of learning is used, suffers from a similar issue---a short window provides few samples to learn from, while a long window provides samples corresponding to significantly different environments.

This paper develops a method that neither assumes that the environment is time-invariant, nor uses intervals of fixed length to artificially adapt the agent to a changing environment while discarding all older learned information. The framework of this paper is one of time-varying Markov decision processes (TVMDPs) as described by \citet{LiuSuk18}, \citet{Lietal19}, and \citet{Ortetal20}: discrete-time, finite-state stochastic control processes with finitely many actions, where transitions from one state to another are governed by a time-varying transition probability function. Building on the maximum-likelihood approach to learning and planning in unknown time-invariant environments \citep{Veretal07,Stretal09,Filetal10}, we propose a \emph{change-conscious maximum likelihood estimate} (CCMLE) that computes a time-varying transition probability function that is \emph{maximally likely}, given (i) the previously observed outcomes of the agent's actions and (ii) a priori fixed bounds on the rate of change of the transition probabilities. In our motivational narrative of an agent on a long-term mission, such bounds may come from prior study of the causes behind possible environmental changes---for example, wind or temperature change \citep{Fenetal05}. An attractive feature of the CCMLE approach is its interpretation as a generalization of a standard estimation method on time-invariant MDPs, described by \citet{StrLit08}. Namely, if the environment is time-invariant the CCMLE matches the estimates provided by the classical method used for time-invariant MDPs.

Using the proposed estimation method, we additionally develop an \emph{active online learning} policy: we define a measure of \textit{uncertainty of a CCMLE}, and by performing actions that seek to minimize the uncertainty, ensure that the agent estimates the system dynamics as quickly as possible during a single system run. We incorporate such a policy into a joint learning and planning mechanism, enabling the agent to perform its task while learning about its unknown and changing environment.


While the framework of TVMDPs has been described by \citet{LiuSuk18}, the work contained therein solely discusses optimal control policies for a priori known TVMDPs. Similarly, \citet{Lietal19} discuss online design of control policies for TVMDPs based on full knowledge of past transition probabilities. The work of \citet{Ortetal20} does consider planning for TVMDPs with unknown transition probabilities and comes closest to our work; however, its underlying estimation method relies on assuming time-invariance. Instead of directly aiming to produce a correct estimate of the TVMDP, it approaches planning by producing a control policy aware of the possible incorrectness in its estimate. While the quantification of incorrectness in the estimate has a similar motivation as our notion of uncertainty in the CCMLEs, it relies on additional assumptions about the TVMDP. Additionally, in contrast with our method of active learning, the approach of \citet{Ortetal20} does not seek to actively reduce this uncertainty during the mission, but instead produces a policy robust to the uncertainty.

In addition to relatively new work on TVMDPs, previous similar frameworks where learning has been discussed include the following:
\begin{itemize}
\item Time-dependent MDPs \citep{BoyLit01}, where the dependence of transition probabilities on time is encoded by appending a continuous time stamp as a coordinate in the state space, thus yielding a continuous-state MDP as defined by \citet{BoyLit01} and \citet{van12}.
\item Time-varying Markov-switching models \citep{Fil94}, which do not include a notion of a control action.
\item Semi-Markov decision processes and related frameworks \citep{Ros70,Sutetal99,YouSim04}, where the transition probabilities themselves are time-invariant, but the time needed to perform a transition may vary.
\item $\varepsilon$-stationary MDPs \citep{Kaletal98,Szietal02,CsaMon08}, which allow for time-varying transition probabilities only inasmuch as they remain close to constant over time.
\end{itemize}
Learning and planning for agents operating in the last two frameworks have been discussed at length \citep{Sutetal99,Szietal02}. However, the nature of these frameworks---with transition probabilities that, potentially disregarding bounded disturbances, do not change over time---yields learning methods that are not useful in the setting where transition probabilities may significantly vary over long periods of time. Learning of time-varying Markov-switching models \citep{Dieetal94} is more similar to the problem of learning for TVMDPs. However, as that framework does not include explicit decision-making, learning simply relies on passively collecting data from multiple system runs. While similar in the estimation part---although with technical differences due to different assumptions on previous knowledge---our proposed method seeks to make the agent actively learn by performing those actions that are expected to reduce the uncertainty in the learned model. Finally, time-dependent MDPs fall into the category of continuous-state MDPs as defined by \citet{van12}. However, learning methods for continuous-state MDPs are often computationally intractable \citep{van12}.

The organization, main contributions, and key theoretical results of this paper are as follows:
\begin{itemize}
    \item Section~\ref{probdef} recalls the definition of a TVMDP and poses problems of optimal learning---during a single system run---and optimal control for an agent operating in a TVMDP with a priori unknown transition probabilities.
    \item Section~\ref{estim} introduces the key element of the proposed learning method: a change-conscious maximum likelihood estimate (CCMLE) of the TVMDP's transition probabilities, given prior observations and knowledge about the rate of change of probabilities.
    \begin{itemize}
        \item Proposition~\ref{gener} shows that CCMLE equals the estimates produced by standard estimation in the case of time-invariant MDPs.
        \item Theorem~\ref{firbig} and Theorem~\ref{secbig} describe the CCMLE in two particular cases of time-varying MDPs.
    \end{itemize}
    \item Section \ref{unc} proposes a measure of uncertainty of a CCMLE.
    \begin{itemize}
        \item Theorem~\ref{unc} geometrically describes the set of all equally likely CCMLEs, given the observations, if such an estimate is not unique.
        \item Theorem~\ref{next} relates the proposed measure to measures of uncertainty used in learning and planning techniques for time-invariant MDPs.
    \end{itemize}
    \item Section \ref{pols} uses the proposed notion of uncertainty to propose optimal learning and control policies for an agent operating in an unknown, time-varying environment. In particular, Section \ref{learnp} considers a policy that minimizes the agent's uncertainty, while Section \ref{contrp} uses the measure of uncertainty as an ``exploration bonus'' in proposing a control policy based on the exploration-exploitation framework, seeking to minimize the uncertainty while directing the agent to progress towards its objective.
    \item Section~\ref{sim} illustrates the developed theory by considering learning and planning for an agent in four scenarios: Section \ref{sim1} discusses a scenario of a one-off change in transition probabilities during a patrolling mission. Section \ref{sim2} considers a setting of regular, periodic changes in transition probabilities on a two-state TVMDP, while Section~\ref{mab} considers a more involved multi-armed bandit setting. Section~\ref{wfe} describes estimation of wind flow using a pilot balloon with CCMLE. In these sections, we compare results attained by the proposed method to the methods introduced in previous work, and show that the proposed method indeed leads to smaller estimation errors and expedited completion of control objectives. 
    \item Proofs of theoretical results are provided in Appendix~\ref{appx}.
\end{itemize}

\noindent\textit{Notation.} Symbol $\NN_0$ denotes all nonnegative integers; $\NN$ denotes all positive integers. Function $d:\RR^n\times\RR^n\to[0,+\infty)$ denotes the Euclidean distance on $\RR^n$. For a set $\cP\subseteq\RR^n$, $\mathrm{diam}(\cP)$ denotes the diameter of the set: $\mathrm{diam}(\cP)=\sup\{d(x,y)~|~x,y\in\cP\}$. For a set $X$, $|X|$ denotes its cardinality. 

\section{Problem Statement}
\label{probdef}

Consider a time-varying Markov decision process (TVMDP) $\cM$ as described by \citet{LiuSuk18}: $\cM=(S,A,P)$, where $S=\{s^1,\ldots,s^n\}$ is the state space, $A$ is the set of actions, and $P:S\times A\times S\times\NN_0\to [0,1]$ is a transition probability function. Namely, $P(s,a,s',t)$ denotes the probability that an agent positioned at state $s\in S$ at time $t\in\NN_0$ will, after performing action $a\in A$, transition to a state $s'\in S$ at time $t+1$. Naturally, $\sum_{s'\in S}P(s,a,s',t)=1$ for all $s\in S$, $a\in A$, $t\in\NN_0$. If $s_0$ is the agent's initial state, the agent's \emph{path} until time $T\in\NN_0$ is denoted by $\sigma=(s_0,\ldots,s_T)$, while the agent's corresponding actions are given by $\alpha=(a_0,\ldots,a_{T-1})$. A \emph{time-varying policy} on a TVMDP $\cM$ is a sequence $\pi=(\pi_1,\pi_2,\ldots)$, where $\pi_t\in A$ may depend on time $t$, the agent's state $s_t$, as well as the agent's previous states $s_0,\ldots,s_{t-1}$. If the transition probability function $P$ is a priori unknown, $\pi_t$ may also depend on the agent's estimate of $P$ at time $t$. A \emph{time-invariant policy} on a TVMDP $\cM$ is a policy that depends solely on the agent's state $s_t$, i.e., with a slight abuse of notation, $\pi:S\to A$.

TVMDPs seek to model an environment in which the agent dynamics may change over time. This framework is a generalization of classical time-invariant Markov decision processes (MDPs); in standard MDPs, transition probability function $P$ is not dependent on time. In the remainder of the paper, if a transition probability is time-invariant, we will denote it by $P(s,a,s',*)$.

We assume that the transition probability function $P$ is unknown to an agent at the beginning of the system run, i.e., prior to $t=0$. As in the previous work on time-invariant MDPs \citep{BraTen02,KeaSin02,KolNg09,FuTop14,Ornetal18}, we study two objectives: 
\begin{enumerate}[(i)]
\item learning the transition probabilities as efficiently and correctly as possible during a single system run, and
\item for a known reward function $R: S\times A\to\RR$, maximizing the agent's expected collected reward over a period of time.
\end{enumerate}

In time-invariant MDPs, because the transition probabilities do not change over time, it is possible to learn the transition probabilities at every state-action pair $(s,a)$ with an arbitrarily small error, by repeatedly visiting the state $s$, performing the action $a$, and observing the outcome. By the law of large numbers, \begin{equation}
\label{lln}
\lim_{\#(s,a)\to\infty}\frac{\#(s,a,s')}{\#(s,a)}=P(s,a,s',*)
\end{equation}
with probability $1$, where $\#(s,a)$ denotes the number of times that action $a$ was performed at state $s$, and $\#(s,a,s')$ denotes the number of times that performing action $a$ at state $s$ led to the agent immediately moving to state $s'$. Hence, under some ergodicity assumptions, it is possible to learn the transition probabilities $P(\cdot,\cdot,\cdot,*)$ within a single system run and with an arbitrarily small error. After learning these probabilities, it is then straightforward \citep{Put05} to determine a policy that comes arbitrarily close to maximizing the agent's expected collected reward, thus solving objective (ii).

In the case of TVMDPs, it is impossible to learn the transition probabilities during a single system run with an arbitrarily small error, as these probabilities may continually change. In fact, if the transition probabilities at different times were entirely independent, the agent would only have one time step (i.e., one action) to learn the transition probabilities at $|S||A|$ state-action pairs. In such a case, any attempt at learning is meaningless. Even if the transition probabilities are not independent, i.e., it is known that there exists $\varepsilon_t\in [0,1)$ such that 
\begin{equation}
\label{cons}
|P(s,a,s',t+1)-P(s,a,s',t)|\leq\varepsilon_t
\end{equation}
for all $s,s'\in S$, $a\in A$ and $t\in\NN_0$,
perfect knowledge of \emph{all} transition probabilities $P(s,a,s',\tau)$ for $\tau<T$ only implies that
$P(s,a,s',T)\in[P(s,a,s',T-1)-\varepsilon_{T-1},P(s,a,s',T-1)+\varepsilon_{T-1}]$ and just one sample collected from $(s,a)$ at time $T$ is not sufficient to correctly determine the value of $P(s,a,s',T)$.

The above discussion behooves us to interpret objective (i) in the following way.

\begin{problem}[Optimal learning in TVMDPs]
Determine a policy $\pi^*$ such that, at every time $t\geq 0$, after taking action $\pi^*_t\in A$, the {\em uncertainty} in the estimated transition probabilities $P(\cdot,\cdot,\cdot,t)$ is minimized.
\end{problem}

We purposefully leave the notion of uncertainty vague at this point. Sections \ref{estim}, \ref{unc}, and \ref{learnp} of this paper will be dedicated to designing a meaningful estimate of transition probabilities, defining the notion of uncertainty of such an estimate, and determining a policy $\pi$ that solves the optimal learning problem.

Largely for notational purposes, we express objective (ii) in terms of expected average rewards on an infinite system run. We emphasize that the reward function itself is assumed to be known.

\begin{problem}[Optimal control in TVMDPs]
\label{proopt}
 Determine a policy $\pi^*$ that maximizes
$$\mathbb{E}\left[\liminf_{T\to\infty}\frac{\sum_{t=0}^T R(s_t,\pi^*_t)}{T}\right]\textrm{,}$$
where $s_t$ is the agent's state at time $t$.
\end{problem}

In both the optimal learning and optimal control problems, we allow $\pi^*_t$ to depend only on the agent's path until time $t$ and its estimates of transition probabilities. In other words, in line with the assumption of a time-varying nature of the environment, we do not allow for learning from repeated runs.

By appending the state space $S$ by the time coordinate and interpreting the transition probabilities and rewards as being defined on the state space $S\times\NN_0$,
the optimal control problem on TVMDP $\cM$ is equivalent to a standard optimal control problem on a countably infinite MDP $\hat{\cM}$, with a finite set of actions $A$ and an averaged reward objective. A detailed discussion of such a problem in the context of infinite MDPs is given by \citet{Put05}. With a slight change to a discounted reward objective, \citet{Put05} shows that such a problem admits a stationary optimal policy, which naturally translates to a time-varying optimal policy on $\cM$. However, such a policy can only be found if the transition probabilities are a priori \emph{known}. Finding an optimal control policy under the stipulation that $P$ is unknown at the beginning of the system run---thus exactly solving the optimal control problem---is impossible. In Section \ref{contrp} we will propose a method motivated by the exploration-exploitation framework of previous work \citep{BraTen02,KeaSin02,KolNg09,FuTop14,Ornetal18}, seeking to actively learn about the transition probabilities in order to be able to collect higher rewards.

We now proceed to discuss the initial building block of our method for learning and planning in TVMDPs: estimating the transition probabilities.

\section{Change-Conscious Maximum Likelihood Estimate}
\label{estim}

The objective of this section is to develop a method for estimating the transition probabilities $P(s,a,s',t)$, $t\leq T$, given the observations of the agent's motion until time $T$. To this end, we develop a {\em change-conscious maximum likelihood estimate} (CCMLE) which produces a set of probability distributions $P(s,a,\cdot,t)$ for all $s\in S$, $a\in A$, and $t<T$, for which the probability of the agent's observed path $\sigma=(s_0,s_1,\ldots,s_T)$ until time $T$ is maximal, given the agent's actions $\alpha=(a_0,\ldots,a_{T-1})$ and a known a priori bound on the rate of change of transition probabilities over time. 

Let us consider the path $\sigma=(s_0,s_1,\ldots,s_T)$. For ease of notation, we assume that $A=\{a\}$, i.e., that the TVMDP $\cM$ is a time-varying Markov chain; if $|A|>1$, we can separate the agent's paths into $|A|$ possibly disconnected paths, one for each action. Given the transition probability function $P$, the probability of the agent following the path $\sigma$ is 
\begin{equation*}
\Prob(\sigma)=\prod_{t=0}^{T-1}P(s_t,a,s_{t+1},t)\textrm{.}
\end{equation*}
As the true transition probability function is unknown, given the agent's path $\sigma$, we want to determine the values $\tilde{P}(s,a,s',t)$ with $s,s'\in S$ and $t<T$ (in future to be denoted by $\tilde{P}_{t<T}$) which are {\em most likely} to have produced such a path \citep{BalNev04}. In other words, we want to find the values of $\tilde{P}_{t<T}$ that maximize $\Prob(\sigma)$. Naturally, we identify $\tilde{P}_{t<T}$ with an element in $[0,1]^{|S|^2T}$. Without any restrictions on the choice of  $\tilde{P}_{t<T}$, such values are naturally given by $\tilde{P}(s_t,a,s_{t+1},t)=1$ for all $t<T$: the transition probability function that will generate the observed outcomes with the highest probability is the one that ensures that all the observed outcomes happen with probability $1$. However, in such a framework, all observations at times $t<T$ make no impact on the estimate of transition probabilities for time $t=T$, rendering any learning meaningless. Thus, we assume the knowledge of the maximal rate of change of transition probabilities, i.e., $\varepsilon_t\in[0,1]$ with $t\in\NN_0$, which satisfy \eqref{cons}. Such a change-conscious maximum likelihood estimation (CCMLE) problem is thus given by 
\begin{equation}
\label{ccmle}
\begin{array}{{>{\displaystyle}c}*2{>{\displaystyle}l}}
\max_{\tilde{P}_{t<T}} & \quad \prod_{t=0}^{T-1} \tilde{P}(s_t,a,s_{t+1},t)  \\
\textrm{s.t.} & \quad \tilde{P}(s,a,s',t)\geq 0 & \quad \textrm{for all } s,s'\in S\textrm{, } t<T\textrm{,} \\
& \quad \sum_{s'\in S} \tilde{P}(s,a,s',t)=1 & \quad \textrm{for all } s\in S\textrm{, } t<T\textrm{.} \\
& \quad |\tilde{P}(s,a,s',t+1)-\tilde{P}(s,a,s',t)|\leq\varepsilon_t & \quad \textrm{for all } s,s'\in S\textrm{, } t<T\textrm{,} \\
\end{array}
\end{equation}
where the decision variables are $\tilde{P}(s,a,s',t)$ for all $s,s'\in S$, $t<T$.

Noting that the product in the objective function of \eqref{ccmle} is nonnegative, and the logarithm function is monotonic, \eqref{ccmle} can be replaced by the constrained log-likelihood problem
\begin{equation}
\label{bigop2}
\begin{array}{{>{\displaystyle}c}*2{>{\displaystyle}l}}
\min_{\tilde{P}_{t<T}} & \quad -\sum_{t=0}^{T-1}\log \tilde{P}(s_t,a,s_{t+1},t)  \\
\textrm{s.t.} & \quad \tilde{P}(s,a,s',t)\geq 0 & \quad \textrm{ for all } s,s'\in S\textrm{, } t<T\textrm{,} \\ & \quad
\sum_{s'\in S} \tilde{P}(s,a,s',t)=1 & \quad \textrm{ for all } s\in S\textrm{, } t<T\textrm{,} \\
& \quad |\tilde{P}(s,a,s',t+1)-\tilde{P}(s,a,s',t)|\leq\varepsilon_t & \quad \textrm{ for all } s,s'\in S\textrm{, } t<T\textrm{,} \\
\end{array}
\end{equation}
with the understanding that $\log 0=-\infty$. 

The optimization problem in \eqref{bigop2} is a convex optimization problem with a linear set of constraints and $T|S|^2$ decision variables $\tilde{P}(s,a,s',t)$; for $|A|>1$, there would be $T|S|^2|A|$ decision variables. Alternatively, as discrete distributions $\tilde{P}(s,a,\cdot,t)$ for different $s$ are not coupled by any of the constraints, we can instead treat \eqref{bigop2} as $|S|$ problems with $T|S|$ decision variables. We will use this method---presented at the beginning of Appendix~A---to simplify computations and theoretical proofs.

The maximal value of the objective function in \eqref{bigop2} is not $+\infty$ because $\tilde{P}_{t<T}$ defined by $\tilde{P}(s,a,s',t)=1/|S|$ for all $s,s'\in S$ and $t<T$ is in the feasible set, and produces a real value for the objective function. Thus, by continuity, the objective function attains a minimum in the feasible set. Such a minimum may not be unique. In the remainder of the paper, we use $\tilde{P}^{T}:S\times A\times S\times\{0,1,\ldots,T-1\}\to [0,1]$ or $\tilde{P}^T_{t<T}$ to denote any CCMLE obtained from the observations until time $T$, i.e., immediately before taking action $a_T$.

The following result, with the proof in Appendix \ref{appx}, shows that the CCMLE directly generalizes the classical estimate from \eqref{lln} for the case of time-invariant transition probabilities.
\begin{proposition}
\label{gener}
Let $\varepsilon_t=0$ for all $t\in\NN_0$. Then, $\tilde{P}^T(s,a,s',*)=\#(s,a,s')/\#(s,a)$ for all $s,s'\in S$, $a\in A$, $T\in\NN_0$, where $\#(s,a,s')=|\{t\in\{0,\ldots,T-1\}~|~s_t=s, a_t=a, s_{t+1}=s'\}|$ and $\#(s,a)=|\{t\in\{0,\ldots,T-1\}~|~s_t=s, a_t=a\}|$.
\end{proposition}

As discussed in Section \ref{probdef}, due to the time-varying nature of transition probabilities, it is generally not possible to ensure that the solution to \eqref{bigop2}, or any other estimation method, indeed correctly estimates the transition probabilities of the TVMDP. Nonetheless, Proposition \ref{gener} shows that, if the transition probabilities are known to be time-invariant, the produced estimates will be asymptotically correct with probability $1$. We now generalize this claim to the case in which the transition probabilities are known to be eventually time-invariant (ETI). ETI systems appear naturally in settings where changes occur on short time scales between long periods of unchanged behavior, e.g., weather fronts \citep{GreBle98}. ETI TVMDPs are also a stochastic discretization of classical ETI control systems \citep{Fei89}.

\begin{theorem}
\label{firbig}
Assume that there exists $N\in\NN_0$ such that $\varepsilon_t=0$ for all $t\geq N$. Then, $$\lim_{\#(s,a)\to\infty}\tilde{P}^T(s,a,s',T-1)=P(s,a,s',T-1)$$ for all $s,s'\in S$, $a\in A$ with probability $1$.
\end{theorem}

The proof of Theorem \ref{firbig} is in Appendix \ref{appx}. Theorem \ref{firbig} states that, asymptotically, the CCMLE will disregard the possible changes in the transition probabilities that occur at the beginning of the system run, \emph{as long as the transition probabilities are known to be time-invariant after a finite time}. Such a property is shared with the estimate $\#(s,a,s')/\#(s,a)\approx P(s,a,s',t)$ which implicitly assumes that the transition probabilities are time-invariant from the start of the system run. 

The following theorem shows that, under the condition that $P(s,a,s',t)=1$ after some $t=N$, the CCMLE actually learns $P(s,a,s',t)$ correctly \emph{in finite time}, as opposed to asymptotically, and without requiring knowledge that $P(s,a,s',t)$ is eventually constant in $t$. We again invite the reader to see Appendix \ref{appx} for the proof.

\begin{theorem}
\label{secbig}
Let $N\in\NN_0$, $s,s'\in S$, and $a\in A$. Assume that $\varepsilon_t=\varepsilon\in(0,1]$ for all $t\in\NN_0$. Let $P(s,a,s',t)=1$ for all $t\geq N$. Then, $\tilde{P}^{T+1}(s,a,s',T)=1$ for all $T\geq N+1/\varepsilon$ such that $(s_T,a_T)=(s,a)$.
\end{theorem}

The slightly convoluted statement of Theorem \ref{secbig} stems from the nature of the optimization problem \eqref{bigop2}.
Namely, if $(s,a)$ is \textit{not} visited at time $T$, any probability distribution $\tilde{P}^{T+1}(s,a,\cdot,T)$ which satisfies the constraints in \eqref{bigop2} can be chosen without any impact on the objective function in \eqref{bigop2}. Such a possibility reflects the agent's lack of knowledge about the drift of transition probabilities at $(s,a)$ since the last time that the agent obtained any information about them.

Theorem \ref{secbig} proves that the CCMLE holds a significant advantage over the estimate given by $\#(s,a,s')/\#(s,a)\approx P(s,a,s',t)$ in the case where a transition probability changes over time and ultimately becomes $1$. Although $\#(s,a,s')/\#(s,a)$ will converge to $1$ as $\#(s,a)\to\infty$, such convergence will be slow: if $(s,a)$ has been visited $v$ times prior to $P(s,a,s',t)$ becoming constantly $1$, it is simple to see that it can take up to $v(1-\eta)/\eta$ additional visits to $(s,a)$ for the estimate of $P(s,a,s',T)$ to have an error no larger than $\eta$, and the error may never become $0$. On the other hand, after a {\em single visit} to $(s,a)$ at time no earlier than $1/\varepsilon$ after $P(s,a,s',t)$ becomes constantly $1$, the estimating procedure \eqref{bigop2} is guaranteed to produce the correct transition probability.

It is possible to modify the classical estimate $\#(s,a,s')/\#(s,a)\approx P(s,a,s',t)$, also used by \citet{Ortetal20}, in order to satisfy Theorem \ref{secbig}: we can simply make the agent ``forgetful''  and calculate the estimate in a sliding window fashion, i.e., based on the outcomes of actions performed in the last $1/\varepsilon$ time steps \citep{Huaetal10,Gajetal18}. In that case, assuming that the transition probabilities satisfy the very specific condition of Theorem \ref{secbig}, the estimate produced in such a way would satisfy the claim of Theorem \ref{secbig}. However, the choice of $1/\varepsilon$ would be arbitrary; an analogous claim to that of Theorem \ref{secbig} would hold for an estimate with any finite memory length, and it is possible that an estimate with a longer or shorter memory would provide better results in general.

The same notion of forgetfulness gives rise to an attractive heuristic to reduce the complexity of computing the CCMLE. Instead of solving an optimization problem with up to $T|S|$ variables at every time step---thus, a problem that grows without bound in size as the system run progresses---we can choose to ``forget'' the variables, i.e., transition probabilities, that are far enough before the current time. \textit{Under very particular conditions}, Theorem \ref{secbig} guarantees that, if we choose to exclude all variables that correspond to time steps that occurred more than $1/\varepsilon$ steps ago, such a \emph{forgetful CCMLE} of $P^T(s,a,s',T-1)$ will not differ from the original CCMLE. In general, forgetfulness is not without effect and we have no reason to believe that the error of the estimates produced by a forgetful CCMLE will be smaller than the one produced by a CCMLE without forgetting. However, the simulations in Section \ref{sim} will show that the difference between a CCMLE and a forgetful CCMLE may be small, while the CCMLE requires significantly more computational power. In particular, for each action the number of decision variables in the forgetful CCMLE is $|S|/\varepsilon$---a value independent of $T$---while in the CCMLE it is $T|S|$.

Having described the CCMLE method for estimating the time-varying transition probabilities, we now continue to the second step in our solution of the optimal learning problem: quantifying how unsure we are about transition probabilities given the agent path.

\section{Uncertainty in Estimation}
\label{unc}

The definition of uncertainty proposed in this paper arises out of two previously identified intuitive causes for uncertainty of an optimal estimate \citep{And04}. Namely, (i) if a single additional observation makes a large difference in an estimate and (ii) if there exist multiple estimates that produce the observed data with the same likelihood. In case (i), while an estimate might be unique, its instability indicates that it is not credible \citep{And04}, while in case (ii) it is uncertain which of the produced estimates, if any, is the correct one. In order to define uncertainty, we first provide a simple description of the set of all solutions to \eqref{bigop2}.

\begin{lemma}
\label{lemunc}
Let the state-action pair $(s,a)$ be visited at times $0\leq T_0<T_1<\ldots<T_k<T$. Then, $\tilde{P}^T(s,a,s_{T_i+1},T_i)$ are uniquely defined.
\end{lemma}

As with nearly all theoretical proofs, the proof of Lemma \ref{lemunc} is in Appendix \ref{appx}. We do provide the proof of the following result within this section, as it provides a method for computation of all CCMLEs.

\begin{theorem}
\label{thmunc}
The set of solutions to \eqref{bigop2} is a polytope.
\end{theorem}
\begin{proof}
By Lemma \ref{lemunc}, the estimates of the transition probabilities $\tilde{P}^T(s_t,a_t,s_{t+1},t)$ obtained from \eqref{bigop2} are uniquely defined for all $t\in\{0,\ldots,T-1\}$. Since the other values in $\tilde{P}^T_{t<T}$ do not feature at all in the objective function, the set of solutions to \eqref{bigop2} is given by any $\tilde{P}^T_{t<T}$ where $\tilde{P}^T(s_t,a_t,s_{t+1},t)$ are the uniquely defined optimal solutions, and all other $\tilde{P}^T(s,a,s',t)$ satisfy the constraints in \eqref{bigop2}. All of those constraints are affine, i.e., the set of solutions to \eqref{bigop2} is a bounded intersection of finitely many halfspaces. Hence, it is a polytope.
\end{proof}

By Lemma \ref{lemunc} and the proof of Theorem \ref{thmunc}, in order to compute the set of all CCMLEs $\cP^T$, we first determine a single solution $\tilde{P}^{T*}_{t<T}$ by solving a convex optimization problem. Then, the set $\cP^T$ of all other solutions $\tilde{P}^T_{t<T}$ is a polytope in $\RR^{|S|^2|A|T}$ given by constraints
\begin{equation}
\label{poly}
\begin{aligned}
\tilde{P}^T(s_t,a_t,s_{t+1},t)=\tilde{P}^{T*}(s_t,a_t,s_{t+1},t) & \quad \textrm{ for all } t<T\textrm{,}  \\
\tilde{P}^T(s,a,s',t)\geq 0 & \quad \textrm{ for all } s,s'\in S\textrm{, } a\in A\textrm{, } t<T\textrm{,} \\
\sum_{s'\in S}\tilde{P}^T(s,a,s',t)=1 & \quad \textrm{ for all } s\in S\textrm{, } a\in A\textrm{, } t<T\textrm{,} \\
|\tilde{P}^T(s,a,s',t)-\tilde{P}(s,a,s',t+1)|\leq\varepsilon_t & \quad \textrm{ for all } s,s'\in S\textrm{, } a\in A\textrm{, } t<T\textrm{.}
\end{aligned}
\end{equation}

We are now ready to define the uncertainty in the estimate $\tilde{P}^T(\cdot,\cdot,\cdot,t)$.

\begin{definition}
\label{dunc}
For all $s\in S$, $a\in A$, $t\leq T$, let $\cP^T_{\sigma,\alpha}(s,a,t)\subseteq\RR^{|S|}$ be the polytope where each point is a probability distribution $\tilde{P}^T(s,a,\cdot,t)$ obtained from the set of all CCMLEs based on the agent's previous trajectory $\sigma=(s_0,\ldots,s_T)$ and actions $\alpha=(a_0,\ldots,a_{T-1})$. The \emph{uncertainty of estimates} $\tilde{P}^T(s,a,\cdot,t)$ under $(\sigma,\alpha)$, denoted by $U^t_{\sigma,\alpha}(s,a)$, is defined by
\begin{equation}
\label{defunc}
U^t_{\sigma,\alpha}(s,a)=\max\left(\max_{s'\in S}\max_{\substack{x\in \cP^T_{\sigma,\alpha}(s,a,t) \\ y\in \cP^{T+1}_{\overline{\sigma}_{s'},\overline{\alpha}}(s,a,t)}}d(x,y),\mathrm{diam}\left(\cP^T_{\sigma,\alpha}(s,a,t)\right) \right)\textrm{,}
\end{equation}
where $\overline{\sigma}_{s'}$ and $\overline{\alpha}$ denote a trajectory and set of actions equal to $\sigma$ and $\alpha$, with an additional transition $(s,a,s')$ observed at time $T$ and an action $a$ performed at the same time.
\end{definition}

When $t=T$, in Definition \ref{dunc} we make use of estimates $\tilde{P}^T(\cdot,\cdot,\cdot,T)$. While \eqref{bigop2} only considered $\tilde{P}^T_{t<T}$, we can introduce additional variables $\tilde{P}^T(\cdot,\cdot,\cdot,T)$ which still need to satisfy the constraints of \eqref{bigop2}. A CCMLE produced by \eqref{bigop2} is then certainly not unique, as $\tilde{P}^T(\cdot,\cdot,\cdot,T)$ can be freely chosen, as long as they respect these constraints. This lack of uniqueness is intuitive: it represents the agent's lack of certainty about the current transition probabilities, even if it has all the possible knowledge about transition probabilities at the previous times. We also note that $\overline{\sigma}$ is not necessarily a legitimate path for an agent, as the agent's state $s_T$ at time $T$ in $\sigma$ does not necessarily equal the starting state for the transition $(s,a,s')$ observed at time $T$. Nonetheless, a CCMLE can be equally produced using  \eqref{bigop2}, with the objective function $$-\sum_{t=0}^{T-1}\log\tilde{P}^{T+1}(s_t,a_t,s_{t+1},t)-\log\tilde{P}^{T+1}(s,a,s',T)\textrm{.}$$

An intuitive explanation of formula \eqref{defunc}, corresponding to the description of uncertainty at the beginning of this section, is as follows. The first term in the $\max$ represents the sensitivity of the CCMLE to a new observation: we are more certain in our knowledge of the transition probabilities if a single ``outlier'' observation cannot significantly change the estimate. The second term represents the distance between two equally likely transition probabilities. If there are two very distant probability transition functions in the polytope $\cP^T_{\sigma,\alpha}(s,a,t)$, this term will be large. We note that $U^t_{\sigma,\alpha}(s,a)\leq \sqrt{2}$, as all probability distributions $\tilde{P}^T(s,a,\cdot,t)$ necessarily belong to the probability simplex, which has a diameter of $\sqrt{2}$.

\begin{remark}
As all sets $\cP^T_{\sigma,\alpha}(s,a,t)$ are polytopes, the latter term in \eqref{defunc} is the maximal distance between the vertices of $\cP^T_{\sigma,\alpha}(s,a,t)$. The first term is the maximum of distances between any point of $\mathcal{P}^T_{\sigma, \alpha}(s,a,t)$ and any point in $\cup_{s'\in S}\mathcal{P}^{T+1}_{\overline{\sigma}_{s'},\overline{\alpha}}(s,a,t)$. The maximum distance between points in any two polytopes is, by convexity, achieved when both of those points are vertices of their corresponding polytopes. Thus, computing the first term also reduces to comparing all distances between vertices of $\mathcal{P}^T_{\sigma, \alpha}(s,a,t)$ and vertices of $|S|$ polytopes $\mathcal{P}^{T+1}_{\overline{\sigma}_{s'},\overline{\alpha}}(s,a,t)$. Hence, $U^t_{\sigma,\alpha}(s,a)$ can be computed by determining vertices of $1+|S|$ polytopes and their pairwise distances.
\end{remark}

The notion of uncertainty emulates the role of functions $1/(1+\#(s,a))$ and $1/\sqrt{\#(s,a)}$ in the setting of time-invariant MDPs, where $\#(s,a)$ denotes the number of times a pair $(s,a)$ has been visited until the current time. In the works of \citet{KolNg09} and \citet{StrLit08}, respectively, those functions---multiplied by a tuning parameter $\beta$---are used to determine which transition probabilities $P(s,a,\cdot,*)$ are not yet known and should be visited. The intuition behind these functions relies on the law of large numbers: as previously discussed, as $\#(s,a)\to\infty$, the estimate $\#(s,a,s')/\#(s,a)$ converges to $P(s,a,s',*)$ with probability $1$, while both of the above functions converge to $0$. Theorem~\ref{next}, proved in Appendix \ref{appx}, shows that $U$ defined in \eqref{defunc} satisfies the same property, and in particular relates $U$ to the function $\beta/(1+\#(s,a))$ used by \citet{KolNg09}.

\begin{theorem}
\label{next}
Assume that $\varepsilon_t=0$ for all $t\in\NN_0$. Let $T\in\NN$. Additionally, let $\sigma$ be the agent's path until time $T$, and let $\alpha$ be the actions that the agent takes until time $T-1$. Then, for all $t\leq T$, $s\in S$, and $a\in A$, $$\frac{\sqrt{1-1/|S|}}{1+\#(s,a)}\leq U^t_{\sigma,\alpha}(s,a)\leq\frac{\sqrt{2}}{1+\#(s,a)}\textrm{,}$$ where $\#(s,a)$ denotes the number of times that $(s,a)$ has been visited until time $T-1$.
\end{theorem}

\begin{remark}
From the proof of Theorem \ref{firbig}, it also directly follows that $U^T_{\sigma,\alpha}(s,a)\to 0$ as $\#(s,a)\to\infty$ for the case of transition probabilities that are known to be ETI. Namely, the proof of Theorem \ref{firbig} shows that, as $\#(s,a)\to\infty$ and thus $T\to\infty$, the set of possible values of $\tilde{P}^T(s,a,s',*_{t\geq N})$ obtained from \eqref{bigop2} at time $T$ is contained in an interval around $P(s,a,s,*_{t\geq N})$ with length converging to $0$. Hence, both terms in \eqref{defunc} converge to $0$. Naturally, the theoretical bounds would be harder to obtain than in Theorem \ref{next} because of the initial steps where the transition probabilities are time-varying.
\end{remark}

We now proceed to the final step of developing an optimal learning and control policy for an agent in a TVMDP: using the uncertainties of state-action pairs to determine which action to perform.

\section{Control Policy Design}
\label{pols}

In designing the optimal control policy, we follow the ``exploration and exploitation'' framework used by previous work \citep{KolNg09,BraTen02,Ornetal18,StrLit08}. In other words, instead of always pursuing actions estimated to be the most useful to satisfying its control objectives (``exploitation''), the agent may take into account the value of a particular action to the learning process, with the goal of minimizing the error in the estimated transition probabilities (``exploration''), hence again helping to achieve the control objectives. We begin by considering an optimal exploration strategy, and then continue to mixing exploration and exploitation.

\subsection{Optimal Learning Policy}
\label{learnp}

We begin by considering solely the learning objective. We interpret the agent's goal as minimization of the total uncertainty $\mathbb{U}(t)=\sum_{(s,a)}U^t_{\sigma,\alpha}(s,a)$ about the environment, where $\sigma$ and $\alpha$ are the sequences of agent's states and actions, respectively, until the beginning of time step $t$. That is, we wish to determine a policy $\pi$ to minimize $\sum_t \mathbb{U}(t)$. Such a sum may be finite, discounted-infinite, or, alternatively, we may only be interested in retaining $\mathbb{U}(t)$ to be as low as possible as $t\to\infty$. 

Analogously to the requirement that arises when attempting to construct an optimal policy in an unknown time-invariant MDP, determining an optimal policy in a TVMDP requires knowledge of the time-varying evolution of $\mathbb{U}$. This evolution depends on the transition probabilities in the underlying TVMDP, which we do not know a priori and are changing over time. 

Faced with the lack of knowledge about the future uncertainties, we propose the following time-varying policy. At time $T$, given a CCMLE of current and future transition probabilities, construct a time-invariant policy \begin{equation}
\label{argpi}
\psi_T^*(s)=\argmin_\psi\mathbb{E}\left[\sum_{t\geq 0}\mathbb{U}_\psi(T+t)\right]\textrm{,}
\end{equation} where $\mathbb{U}_\psi$ denotes the uncertainty if the agent follows a time-invariant policy $\psi$ starting at $s$, and the expectation is taken with the assumption that the current CCMLE of $P$ is correct. Then, define $\pi_T=\psi_T^*(s_T)$. 

Naturally, there are no guarantees that the proposed policy $\pi$, consisting of $\pi_T$'s generated at each step, will indeed minimize the total uncertainty about the system. Policy $\pi$ represents a heuristic attempt to always follow whichever action seems to be optimal at reducing the future uncertainty. We note that at every time $T$ it depends on the current CCMLE of transition probabilities, including probabilities $P(\cdot,\cdot,\cdot,T+t)$ for $t\geq 0$. As discussed before, such a CCMLE may not be unique. We also note that we are being non-committal about the horizon length (or possible discounts) for the sum in \eqref{argpi}. Its choice depends on the horizon of our learning objective; we reserve further discussion for the subsequent section.

\subsection{Optimal Control Policy}
\label{contrp}

We now amend the above discussion about optimal active learning by considering an agent that desires to maximize the collected state-action-based rewards. This framework is the setting of \citet{BraTen02}, \citet{KeaSin02}, and \citet{KolNg09}, which deal with the combined exploration and exploitation objectives---learning to improve the accuracy of the estimated transition probabilities (and thus lead to better planning in the future) and attempting to collect rewards using the current estimates---by adding a \textit{learning bonus} to the agent's collected reward. In other words, instead of using the policy \begin{equation*}
\pi^*_T=\argmax_\pi\mathbb{E}\left[\sum_{t\geq 0} R(s_{T+t},\pi_{T+t})\right]\textrm{,}
\end{equation*}
the agent uses the policy \begin{equation*}
\pi^*_T=\argmax_\pi\mathbb{E}\left[\sum_{t\geq 0} R(s_{T+t},\pi_{T+t})+f(s_{T+t},\pi_{T+t},{T+t})\right]\textrm{,}
\end{equation*}
where $f(t)$ relates to the ``amount of information'' that the agent will collect by visiting $(s_t,\pi_t)$ at time $t$. As mentioned, \citet{KolNg09} define $f(s,a,t)=\beta/(\#(s,a)+1)$, where $\#(s,a)$ denotes the number of times $(s,a)$ has been visited before time $t$. The results of \citet{KolNg09} show that, for time-invariant MDPs, such a bonus will, with high probability, lead to eventual learning of an almost-optimal policy in the Bayesian sense, as defined by \citet{KolNg09}. Our framework does not allow for such a result, as there is a constant need for learning due to the change in transition probabilities. Nonetheless, we adapt the approach of \citet{BraTen02}, \citet{KeaSin02}, and \citet{KolNg09} and define an optimal control policy to be a policy \begin{equation}
\label{defopt}
\pi^*_T=\argmax_\pi\mathbb{E}\left[\sum_{t\geq 0} R(s_{T+t},\pi_{T+t})+\beta U_{T+t}(s_{T+t},\pi_{T+t})\right]\textrm{,}
\end{equation}
where $\beta\geq 0$ and $U_t$ denotes the uncertainty of the agent about the estimates of current transition probabilities associated with the particular state-action pair, given the agent's motion and actions until time $t$. As shown in Theorem \ref{next}, in the case of time-invariant MDPs, the bonus defined in \eqref{defopt} is indeed similar to the form of the bonus of \citet{KolNg09}. We again note that the expectation in \eqref{defopt} is computed using the CCMLE at time $T$. Thus, as in the previous section, the agent needs to recompute and reapply the proposed policy during the system run in order to make use of its learning.

Naturally, policy proposed in \eqref{defopt} depends on the length of the horizon that is considered; \citet{KolNg09} provide a theoretical discussion of an optimal horizon length for the policy used in that work. As computing predictions of the future uncertainties is computationally difficult, solving \eqref{defopt} comes with a heavy computational burden for long horizons. Predictions of future uncertainties also become increasingly unreliable with the horizon length, so $\mathbb{E}[U_t(s_t,\pi_t)]$ may not be meaningful for large $t$. For this reason, in the work of Section \ref{sim} we concentrate on computing policy \eqref{defopt} with (i) horizon $1$ or (ii) $\beta=0$, the latter of which does not promote active learning, but nonetheless enables the agent to improve its estimates of transition probabilities by observing new transitions.

\section{Simulations}
\label{sim}

In this section, we illustrate the proposed CCMLE method on several numerical examples. The first example is within the classical gridworld-based patrolling (or pickup-delivery) domain \citep{Sanetal04,Cheetal12,Toretal18}, where probabilities of the agent moving in a particular direction when using a given action change over time. The second example is that of a two-state MDP with periodically changing transition probabilities. In addition to learning transition probabilities, we consider a reward maximization objective. The third example is a multi-armed bandit problem, where the rewards on the arms periodically vary over time. Finally, the fourth example considers a wind flow estimation scenario, and illustrates the use of CCMLE when the environment changes in a partly random fashion, and bound $\varepsilon$ is poorly chosen. Table \ref{tab1} lists all the performed experiments.
\begin{table}[t]
    \centering
    \begin{tabular}{|p{15mm}|p{40mm}|p{15mm}|p{25mm}|p{35mm}|}
    \hline
     \textbf{Section} & \textbf{Change in transition probabilities } & \textbf{Task} & \textbf{Bound \newline correctness} & \textbf{Method} \\
     \hline
     \ref{ex1l} & Constant, then \newline time-invariant & Learning & Correct & CCMLE \\
     \hline
     \ref{ex1l} & Constant, then \newline time-invariant & Learning & Incorrect & CCMLE \\  
     \hline
    \ref{ex1l} & Extreme change in one time step, then \newline time-invariant & Learning & Incorrect & CCMLE \\  
    \hline
    \ref{ex1p} & Constant, then \newline time-invariant & Planning & Correct & CCMLE, no active learning in planning \\  
    \hline
    \ref{ex2l} & Periodic & Learning & Correct & CCMLE \\
    \hline
    \ref{ex2l} & Periodic & Learning & Correct & Forgetful CCMLE \\
    \hline
    \ref{ex2p} & Periodic & Planning & Correct & CCMLE, no active learning in planning \\  
    \hline    
    \ref{ex2p} & Periodic & Planning & Correct & CCMLE \newline + active learning \\  
    \hline
    \ref{mab} & Periodic & Planning & Correct, \newline but loose & Forgetful CCMLE \newline + active learning \\  
    \hline
    \ref{wfe} & Random & Learning & Correct & CCMLE \\  
    \hline
    \ref{wfe} & Random & Learning & Incorrect & CCMLE \\  
    \hline
    \end{tabular}
    \caption{Overview of numerical examples of Section \ref{sim}.}
    \label{tab1}
\end{table}

\subsection{Patrol with ETI Dynamics}
\label{sim1}

The scenario we are simulating is as follows: an agent is moving on an $n\times n$ grid, starting in one of the grid corners. A $5\times 5$ illustration is shown in Figure~\ref{grid}. At every time step, if the agent is at a non-edge tile in the grid, it is known that it will move north, east, south, west, or stay in place. Upon hitting a grid edge (``wall''), the agent is known to automatically ``bounce back'', i.e., move to the nearest non-edge tile in the subsequent step.

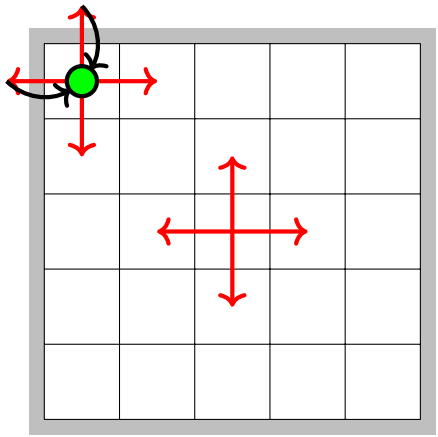
\begin{figure}[t]
\centering
\begin{tikzpicture}[scale=1]
\fill[gray!50!white] (-0.2,-0.2) rectangle (0,5);
\fill[gray!50!white] (-0.2,-0.2) rectangle (5,0);
\fill[gray!50!white] (-0.2,5) rectangle (5.2,5.2);
\fill[gray!50!white] (5,-0.2) rectangle (5.2,5.2);
\draw[step=1cm] (0,0) grid (5,5);
\draw[ultra thick, red, ->]  (2.5,2.5)--(2.5,3.5);
\draw[ultra thick, red, ->]  (2.5,2.5)--(2.5,1.5);
\draw[ultra thick, red, ->]  (2.5,2.5)--(1.5,2.5);
\draw[ultra thick, red, ->]  (2.5,2.5)--(3.5,2.5);
\draw[ultra thick, red, ->]  (0.5,4.5)--(1.5,4.5);
\draw[ultra thick, red, ->]  (0.5,4.5)--(0.5,3.5);
\draw[ultra thick, red, ->]  (0.5,4.5)--(-0.5,4.5);
\draw[ultra thick, red, ->]  (0.5,4.5)--(0.5,5.5);
\draw[ultra thick, black, ->]  (-0.5,4.5) arc(225:300:0.707);
\draw[ultra thick, black, ->]  (0.5,5.5) arc(45:-30:0.707);
\draw[ultra thick, fill=green] (0.5,4.5) circle (0.2);

\end{tikzpicture}
\caption{An illustration of the grid world in the simulated scenario. The walls are denoted in gray. The possible motions of an agent during one time step are denoted in red; the agent can also remain in place. If the agent's action results in the agent moving into a wall, the agent automatically returns to its last non-wall position in the subsequent time step (denoted in black).}
\label{grid}
\end{figure}

The agent can use one of $5$ actions at every point in time: $A=\{1,2,3,4,5\}$. To reduce the computational complexity of the calculations and fit in line with previous work on planning for similar tasks \citep{Toretal18}, the transition probabilities at every non-edge tile in the grid are known to the agent to be the same; however, the probabilities themselves are not known to the agent. At the beginning of the system run, if the agent is at a non-edge tile, action $1$ results in the agent moving north, $2$ in moving east, $3$ in moving south, $4$ in moving west, and $5$ in remaining in place. However, during the system run, actions $1$ and $3$ are slowly switching their outcomes, and so are actions $2$ and $4$. More precisely, at time $t$, action $1$ ($2$, $3$, $4$, respectively) will result in the agent moving north (east, south, west, respectively) with probability $1-t/100$ for $t\leq 100$ and $0$ for $t>100$ and in the agent moving south (west, north, east, respectively) with probability $t/100$ for $t\leq 100$ and $1$ for $t>100$.

We consider two settings: in the first one, the agent solely seeks to learn the transition probabilities, while in the second one, it seeks to satisfy a patrolling objective.

\subsubsection{Learning}
\label{ex1l}

In this task, the agent's sole goal is to learn the transition probabilities. Its control action is always the action that has been least used so far in the system run. In this example, we compare two basic ways of agent's learning: (i) by performing classical estimation---assuming that the transition probabilities are time-invariant, counting outcomes, and dividing by the number of times that an action was taken---and (ii) by using the knowledge that the change in transition probabilities between consecutive time steps is no larger than $0.01$ and obtaining the CCMLE. While method (i), also used by \citet{Ortetal20}, will lead the agent to converge to the correct transition probabilities (as they are time-invariant after $t=100$), such convergence is only asymptotic. On the other hand, the CCMLE method in (ii) takes into account the possible change in transition probabilities, and implicitly quickly rejects those samples that were collected much earlier during the system run. The average error $\sum_{s,a,s'}|\tilde{P}(s,a,s',t-1)-P(s,a,s',t-1)|/(|S|^2|A|)$ in estimated transition probabilities at time $t$ is given in Figure~\ref{avgerr}(a), and the maximal error $\max_{s,a,s'}|\tilde{P}(s,a,s',t-1)-P(s,a,s',t-1)|$ is given in Figure~\ref{avgerr}(b). We note that all the results in this section correspond to the $5\times 5$ grid. However, since the probabilities of moving in a particular direction in the grid do not depend on the state, the results for grids of all sizes are qualitatively the same.

\begin{figure}[t]
\centering
\includegraphics[width=0.45\textwidth]{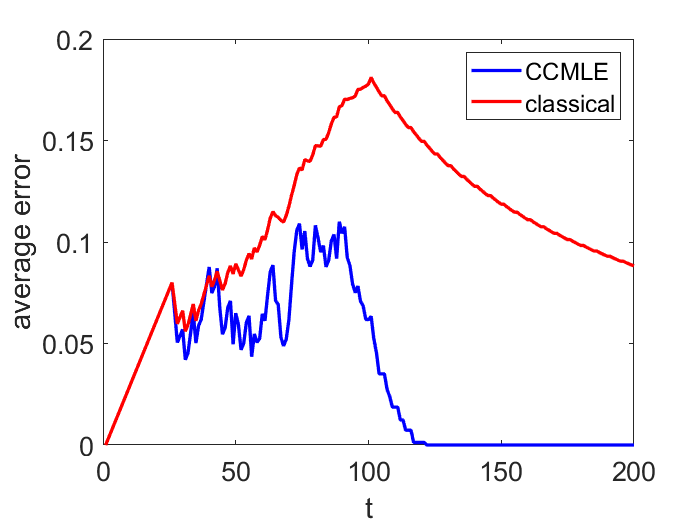}
\includegraphics[width=0.45\textwidth]{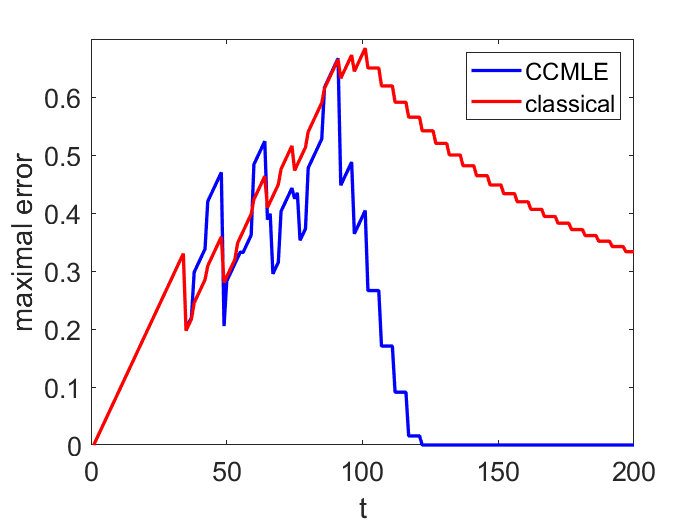}
\includegraphics[width=0.45\textwidth]{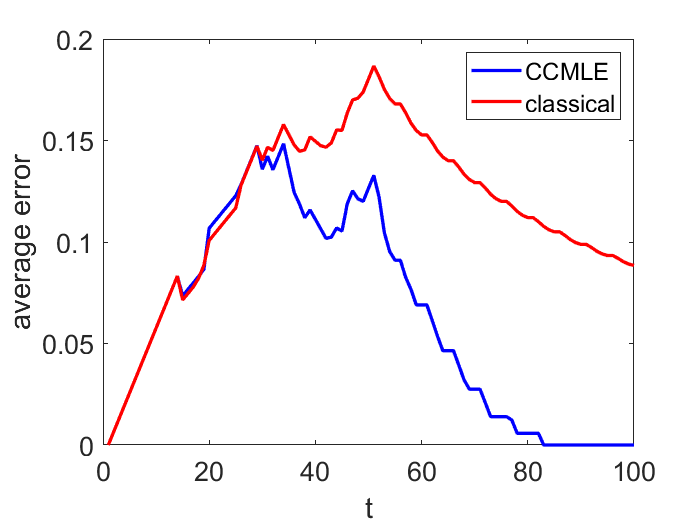}
\includegraphics[width=0.45\textwidth]{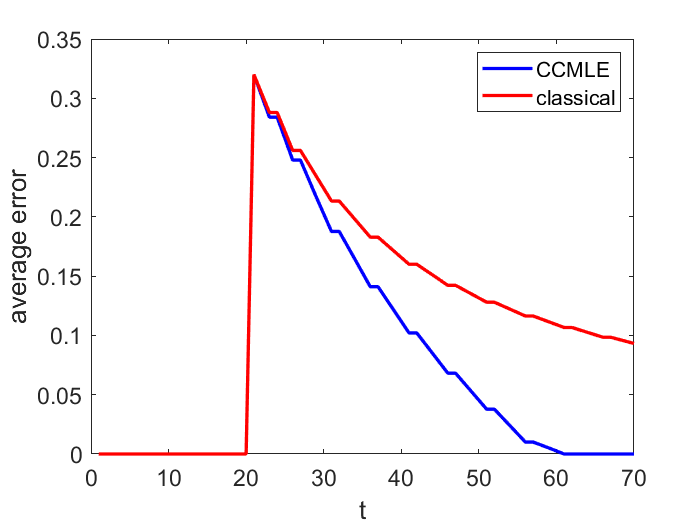}

\caption{(a) Average error in the transition probability estimates. (b) Maximal error in the transition probabilities estimates. (c) Average error in the transition probability estimates, in the scenario where the rate of change is twice higher than expected. (d) Average error in the transition probability estimates, in the scenario with a single-step immediate switch in transition probabilities.}
\label{avgerr}
\end{figure}

The CCMLE method produces significantly better results than classical estimation. Near the start of the system run, the importance of changes over time is small, and the two methods perform similarly. As the transition probabilities continue changing, classical estimation is unable to adapt, and the error in the estimated transition probabilities continues growing. While the CCMLE method is also unable to provide entirely correct estimates, its estimates are on average better than the classical one, and, at around $t=90$, i.e., even before the transition probabilities stop changing, both the average and the maximal errors begin to quickly decrease. While Theorem \ref{secbig} only guarantees that the CCMLE method will produce correct estimates of transition probabilities at time $t=200$, this result already occurs around $t=120$, i.e., only $20$ time steps after the transition probabilities cease changing. On the other hand, classical estimation continues having a comparatively large (albeit diminishing) average and maximal error. 

To display the power of CCMLE, we end the task by considering a scenario not covered by the theoretical work of previous sections: the transition probabilities changing more rapidly than the bound $\varepsilon_t$. In particular, we retain the above scenario, and the agent still counts on the transition probabilities changing by no more than $0.01$ between two consecutive steps. However, the transition probabilities proceed to change more quickly. In the first case, the probabilities shift by $0.02$ instead of $0.01$, thus completing the shift by time $t=50$. In the second case the transition probabilities shift by $1$ in a single step: instead of shifting in slow increments as before, the switch in outcomes between actions $1$ and $3$, as well as actions $2$ and $4$, happens instantaneously at time $t=20$. Figures~\ref{avgerr}(c) and Figure~\ref{avgerr}(d) illustrate the average error in the estimated transition probabilities at time $t$ for these two cases, respectively, again comparing classical estimation unconscious of possible changes in transition probabilities and the CCMLE. In both cases the CCMLE outperforms the classical estimation. In particular, in the latter, more extreme scenario, the CCMLE method obtains an entirely correct model around $40$ time steps after the sudden change. The estimated transition probabilities remain correct at all times afterwards. On the other hand, while classical estimation converges towards the correct transition probabilities, it does so slowly. For instance, $50$ time steps after the switch, the error is still greater than the error of the CCMLE at $25$ time steps after the switch, and the error of the classical estimation never reaches $0$. Table \ref{tab2} compares the mean errors over time of the CCMLE with those of classical estimates.

\begin{table}[b]
    \centering
    \begin{tabular}{|c|c|c|}
    \hline
     & \multicolumn{2}{c|}{\textbf{Mean error}} \\
    \hline
    \textbf{Experiment} & \textbf{CCMLE} & \textbf{Classical estimation} \\
    \hline
    (a) & 0.03 & 0.11  \\
    \hline
    (b) & 0.18 & 0.41 \\
    \hline
    (c) & 0.06 & 0.12 \\
    \hline
    (d) & 0.1 & 0.16 \\
    \hline
    \end{tabular}
    \caption{Mean errors for experiments described in Figure \ref{avgerr}. The mean error in (d) only considers time steps from $t=21$ to $t=70$, as there are no errors before the change at $t=20$.}
    \label{tab2}
\end{table}

\subsubsection{Planning}
\label{ex1p}
In this task, the agent seeks to satisfy the following control objective: reach the eastern wall of the grid, then reach the southern wall, then the western wall, then the northern wall, in this order, and repeat the process indefinitely. Such an objective is a patrolling task in the sense of \citet{Toretal18}, and is a version of the pickup-delivery objective as defined by \citet{Cheetal12}, where there are multiple pickup and delivery points. We consider the setting of experiments (a) and (b) described in Figure \ref{avgerr}, i.e., transition probabilities slowly changing over $100$ time steps, with $\varepsilon=1/100$.

If we encode the described control objective into a reward function, the rewards that the agent obtains are time-varying, i.e., depend on the agent's previous path. While such a framework technically differs from the setting of the previous sections, the agent is still able to compute the optimal policy \eqref{defopt} at every time, given the rewards at that time, and then recompute it once the rewards change. In our framework, the agent uses the policy \eqref{defopt} with $\beta=0$---thus, without any conscious exploration effort---and with a horizon long enough to ensure that it has the incentive to visit the desired wall as soon as possible. Figure~\ref{walrea} illustrates how often the agent is able to reach its objective wall. The two methods enjoy a comparable success rate at the beginning. However, the agent that learned using the CCMLE is able to adapt to the changes in transition probabilities much more quickly than the agent that uses a classical estimation method. By time $t=200$, the agent using CCMLE visits $38$ walls, while the agent that uses classical estimation visits $18$ walls.

\begin{figure}[t]
\centering
\includegraphics[width=0.55\textwidth]{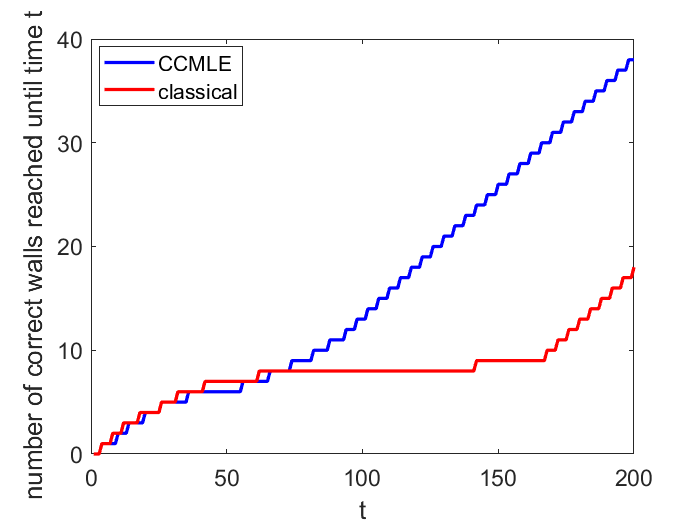}
\caption{Satisfaction of control objectives. The red curve and blue curve indicate, respectively, the number of times the agent reached a correct wall until time $t$ when using the plan that assumes time-invariant transition probabilities and when using the CCMLE-based plan conscious of changes in dynamics.}
\label{walrea}
\end{figure}

While the above simulation, along with other simulations in this section, does not present the effect of different initial states on the CCMLE and subsequent learning and control policies, we remark that such an effect is not difficult to deduce---at least on an informal level---given the time-varying nature of dynamics. Namely, for a state space where an agent may not visit a particular state-action pair for a substantial number of steps (``sparse state space''), the estimate of probabilities at states that have not been recently visited will necessarily be incorrect using \textit{any} estimation method, as the dynamics will have changed since the last visit. Indeed, the proof of Theorem~\ref{thmunc} indicates that the CCMLE method has no knowledge about the transition probabilities $P_t(s,a,\cdot)$ of a state-action pair $(s,a)$ after the last time $t$ it has been visited, and can equally choose any probabilities consistent with its bounds on the rate of change of probabilities and its estimate of $P_t(s,a,\cdot)$. Thus, for a ``sparse state space'', the agents with different initial states are more likely to have---at least at the beginning of their run---visited different parts of the state space, and thus will have meaningful estimates of transition probabilities only at different parts of the state space. Their control policies may thus substantially differ. On the other hand, for a ``dense state space'' where an agent frequently visits each state-action pair regardless of its starting state, the estimates will be similar for all initial states. Two agents situated at the same state at the same time will thus possibly take the same action, even if their initial states were different. Following this discussion, we omit in-depth discussions of different initial states from the remainder of the section.

\subsection{Periodically Changing Transition Probabilities}
\label{sim2}

The scenario we simulate in this section is that of a $2$-state TVMDP illustrated in Figure~\ref{2mdp}. As in Section \ref{sim1}, we first consider solely estimation, i.e., learning, and then joint learning and planning. Since the TVMDP given in Figure~\ref{2mdp} only contains a single action, for the discussion of planning we will append an additional action.

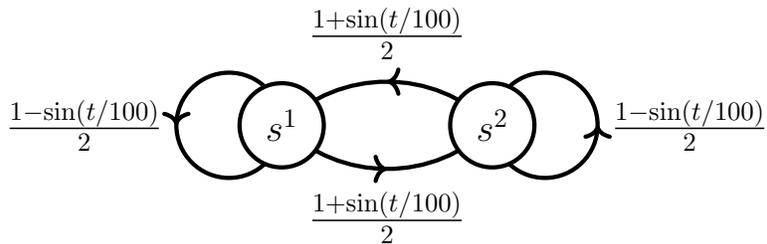
\begin{figure}[t]
\centering
\begin{tikzpicture}[scale=0.7]
\tikzset{->-/.style={decoration={
  markings,
  mark=at position .5 with {\arrow{>}}},postaction={decorate}}}
    	\draw[ultra thick,black,->-] (0,0) arc(225:315:2.83);
     	\draw[ultra thick,black,->-] (4,0) arc(45:135:2.83);
     	\draw[ultra thick,black,->-] (0,0) arc(0:360:1);
     	\draw[ultra thick,black,->-] (4,0) arc(180:540:1);
        \node at (2,1.75) {\Large $\frac{1+\sin(t/100)}{2}$};
        \node at (2,-1.75) {\Large $\frac{1+\sin(t/100)}{2}$};
        \node at (-3.75,0) {\Large $\frac{1-\sin(t/100)}{2}$};
        \node at (7.75,0) {\Large $\frac{1-\sin(t/100)}{2}$};        

\draw[ultra thick, fill=white] (0,0) circle (0.8cm);
\draw[ultra thick, fill=white] (4,0) circle (0.8cm);
        \node at (0,0) {\Large $s^1$};
        \node at (4,0) {\Large $s^2$};
\end{tikzpicture}
\caption{The scenario of periodically changing transition probabilities. The time-varying transition probabilities, a priori unknown to the agent, are indicated next to the arrows indicating transitions.}
\label{2mdp}
\end{figure}

\subsubsection{Learning}
\label{ex2l}

As in the previous section, we compare classical estimates, produced by assuming that the transition probabilities are time-invariant, to the CCMLE. In particular, we assume that it is a priori known that the transition probabilities change by no more than 
\begin{equation}
\label{vare}
\varepsilon=\frac{1}{2}\left(1-\cos\frac{1}{100}+\sin\frac{1}{100}\right)\textrm{.}
\end{equation}
However, the agent does not know the exact change in the transition probabilities, nor is it aware that the transition probabilities are periodic.

Figure~\ref{sinest}(a) shows the estimate of the probability of a \textit{switch}: a transition that moves the agent from $s^1$ to $s^2$ or vice versa. Classical estimation that assumes time-invariant transition probabilities obviously produces estimates that converge towards the mean of the time-varying transition probability. Thus, the obtained estimates are increasingly less accurate as time progresses. On the other hand, the CCMLE tracks the true transition probability with remarkable accuracy.

\begin{figure}[t]
\centering
\includegraphics[width=0.45\textwidth]{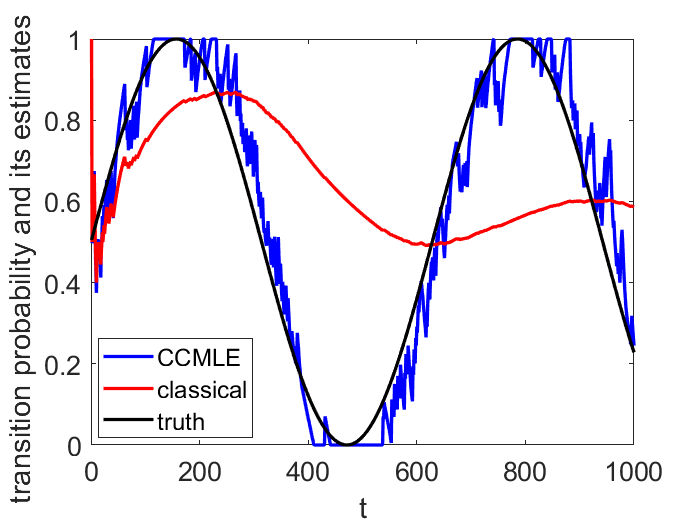}
\includegraphics[width=0.45\textwidth]{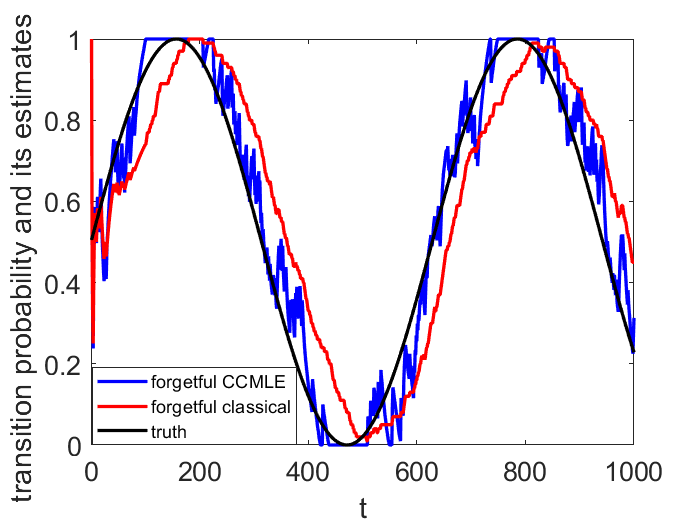}

\caption{(a) Estimates of the transition probability which results in a switch. (b) Estimates of the transition probability which results in a switch, with forgetful estimation. The black curves indicate the true probability.}
\label{sinest}
\end{figure}

As mentioned in Section \ref{estim}, classical learning can be heuristically made aware of the changes in the transition probabilities by using a sliding window, i.e., counting only the samples obtained within finitely many previous steps. On the other hand, making the CCMLE method---which is already explicitly aware of the bound on the changes of transition probabilities---forgetful can be useful to reduce the complexity of the relevant optimization problems. The results of Section \ref{estim} indicate that $1/\varepsilon$ may be an appropriate amount of memory. Figure~\ref{sinest}(b) gives the comparison between the two forgetful learning methods. The classical estimation method, now made forgetful, identifies the transition probabilities with a delay---by the time it collects enough samples about a particular transition probability, the probability has changed. On the other hand, introduction of forgetfulness into the CCMLE does not result in a significant impact on its quality; in fact, as shown in Table~\ref{tab3}, it results in a slight reduction of its error.

The CCMLE, with or without forgetfulness, still outperforms the classical estimation, even when the classical method is improved by introducing forgetfulness. Table \ref{tab3} compares the mean errors over time of the CCMLE, forgetful CCMLE, classical, and forgetful classical estimates.

\begin{table}[b]
    \centering
    \begin{tabular}{|c|c|c|c|}
    \hline
    \textbf{CCMLE} & \textbf{Forgetful CCMLE} & \textbf{Classical} & \textbf{Forgetful classical} \\
    \hline
    0.06 & 0.05 & 0.28 & 0.14 \\
    \hline
    \end{tabular}
    \caption{Mean errors for experiments described in Figure \ref{sinest}.}
    \label{tab3}
\end{table}

\subsubsection{Planning}
\label{ex2p}

In order to discuss the optimal control policy, we now slightly modify the working TVMDP, as shown in Figure~\ref{2mdpb}. Our new setting introduces a single deterministic action ($black$) available at state $s^2$. We define the reward for such an action by $R(s^2,black)=3$. State $s^1$ admits two available actions: action $blue$ is deterministic and we define $R(s^1,blue)=1$, while action $red$ may end up in two outcomes, as shown in Figure~\ref{2mdpb}, and its reward is given by $R(s^1,red)=0$. The agent's starting state is $s^1$.

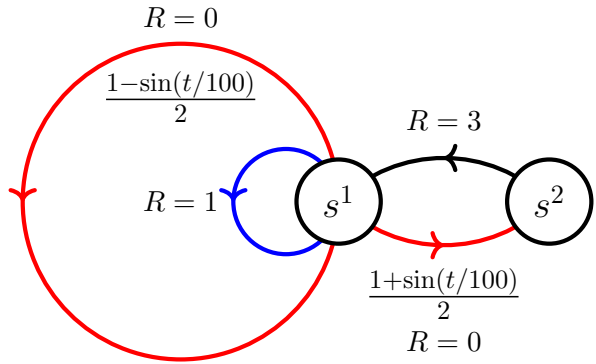
\begin{figure}[t]
\centering
\begin{tikzpicture}[scale=0.7]
\tikzset{->-/.style={decoration={
  markings,
  mark=at position .5 with {\arrow{>}}},postaction={decorate}}}
    	\draw[ultra thick,red,->-] (0,0) arc(225:315:2.83);
     	\draw[ultra thick,black,->-] (4,0) arc(45:135:2.83);
     	\draw[ultra thick,blue,->-] (0,0) arc(0:360:1);
     	\draw[ultra thick,red,->-] (0,0) arc(0:360:3);
        \node at (2,-1.75) {\Large $\frac{1+\sin(t/100)}{2}$};
        \node at (-3,2) {\Large $\frac{1-\sin(t/100)}{2}$};
        \node at (-3,3.5) {$R=0$};
        \node at (-3,0) {$R=1$};
        \node at (2,-2.65) {$R=0$};
        \node at (2,1.55) {$R=3$};

\draw[ultra thick, fill=white] (0,0) circle (0.8cm);
\draw[ultra thick, fill=white] (4,0) circle (0.8cm);
        \node at (0,0) {\Large $s^1$};
        \node at (4,0) {\Large $s^2$};
\end{tikzpicture}
\caption{The scenario of periodically changing transition probabilities, with control actions. The transition probabilities when using the black and blue actions are $1$; those actions are deterministic. The transition probabilities when using the red action are as indicated.}
\label{2mdpb}
\end{figure}

We note that the described setting corresponds to the two-armed semi-Markov bandit problem in the sense of \citet{Tsi94} and \citet{DufBar97}; we omit a more detailed description of the bandit, but consider a related multi-armed bandit problem in the subsequent example.

Throughout this example, for simplicity we assume that the agent is a priori aware that the actions $blue$ and $black$ are deterministic. Hence, its uncertainty is only about action $red$; again, the agent knows that the rate of change does not exceed $\varepsilon$ given in \eqref{vare}. The problem of computing the agent's uncertainty, as given in Definition \ref{dunc}, can be made computationally more simple owing to the fact that action $red$ only has two possible outcomes: agent's total uncertainty and the uncertainty in the estimated probability of a switch are scalar multiples, the former being larger by a factor of $\sqrt{2}$.

In order to maximize its average collected reward, the agent applies the optimal policy given in \eqref{defopt}, with the horizon length equal to $1$, and recomputes and reapplies it at every time step. In other words, at every time the agent cares only about the results of its current and next step. We compare the agent's average reward between three cases: (a) when the agent bases its decision on the estimates obtained by classical estimation, (b) when the agent uses CCMLE, but does not use active learning, and (c) when the agent uses CCMLE with active learning, i.e., $\beta>0$.

An agent that uses learning based on classical estimation, with or without a learning bonus as used by \citet{KolNg09} and \citet{StrLit08}, would choose action $red$ when its estimate of transitioning to state $s^2$ with the red action is greater than $2/3$, or it receives a sufficient bonus, and choose action $blue$ otherwise. Such an agent will always eventually choose action $blue$: learning bonuses will eventually converge to $0$, and the estimate of the probability of a switch for an agent that uses classical learning eventually converges to $1/2<2/3$. Thus, the average reward of an agent that assumes that probabilities are not changing will converge to $1=R(s^1,blue)$.

An agent that uses the CCMLE  without active learning, i.e., applies \eqref{defopt} with $\beta=0$, would essentially fall into the same trap. Once its estimate of the probability of a switch falls under $2/3$ once---which is likely to happen, due to the oscillating nature of the transition probabilities and the good quality of estimation exhibited by a CCMLE-based learner in previous examples---the agent will again cease to vary its actions, and will use solely action $blue$. Thus, its average reward will converge to $1$.

If $\beta>0$, the agent may choose action $red$ in order to reduce its uncertainty. Thus, even after its estimate of the probability of a switch falls under $2/3$, it may still occasionally perform action $red$, thus allowing itself to observe the changed transition probabilities and restart collecting higher rewards in the periods when the transition probability is higher than $2/3$. An example of such behavior is exhibited in Figure~\ref{avgrew}, for $\beta=3/\sqrt{2}$.

\begin{figure}[t]
\centering
\includegraphics[width=0.55\textwidth]{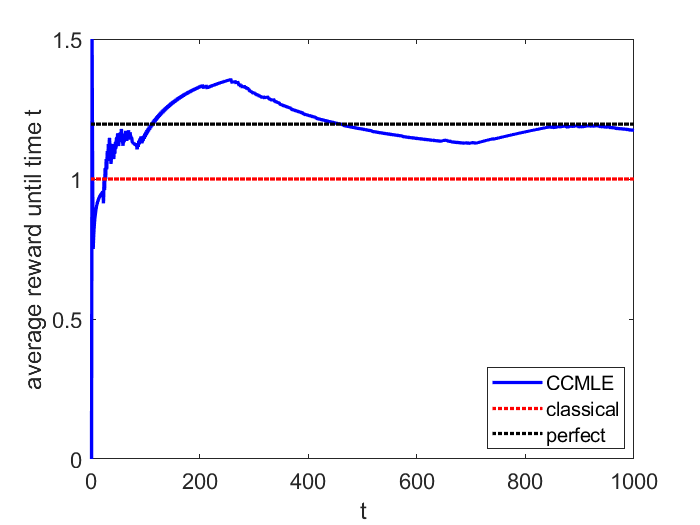}
\caption{Rewards obtained by an agent in a two-state MDP. The blue line indicates the average reward obtained by an agent who uses the CCMLE with an uncertainty-based learning bonus. The red line indicates both the asymptotic average reward obtained by using classical estimation unconscious of the change in the transition probabilities and that of the CCMLE without active learning. The dotted black line indicates the asymptotic average reward obtained by an agent who has perfect knowledge of the transition probabilities.}
\label{avgrew}
\end{figure}

As Figure~\ref{avgrew} shows, the average reward obtained by an agent who uses policy \eqref{defopt} converges to around $1.15$. This is a significant improvement over the average reward of $1$: an agent with perfect knowledge who chooses to use action $red$ whenever the transition probability of a switch is greater than $2/3$ and action $blue$ otherwise will obtain an asymptotic average reward of around $1.2$. Even so, slightly higher gains may be obtained by a different choice of $\beta$. We note, however, that increasing $\beta$ does not simply result in a higher reward. For high enough $\beta$ the agent will always choose action \textit{red} in order to reduce uncertainty. Numerical results show that such an agent will obtain an asymptotic average reward of $0.88$; worse than rewards obtained for both $\beta=0$ and $\beta=3/\sqrt{2}$.

\subsection{Multi-Armed Bandit}
\label{mab}

Motivated by the two-armed bandit of the previous section, in this example, we consider a standard $n$-armed bandit setting \citep{Tsi94}. At each time step, the agent pulls one arm of a bandit. Each arm pull lasts a single time step, and produces a different reward depending on the arm. Arm $1$ always produces a reward of $1$. For each $i$, $i\in\{2,\ldots,n\}$, arm $i$ produces one of the two rewards at time $t$: a reward of $i$ with probability $0.95(\sin(\gamma_it+\delta_i)+1)/i$, and a reward of $0$ with probability $1-0.95(\sin(\gamma_it+\delta_i)+1)/i$. While the possible rewards are known to the agent, their probabilities are not.

The considered setting does not immediately follow the framework of the remainder of this paper; namely, there are no transitions between states, and the agent's collected reward is unknown in advance. Nonetheless, we are able to convert such a framework into a TVMDP by introducing states $s_0, \ldots, s_{n+1}$. The actions available at state $s_0$ are the bandit arms. Pulling arm $1$ results in the agent moving to state $s_1$. Pulling arm $i\geq 2$ results in the agent moving to either the state $s_i$ or $s_{n+1}$, the former with probability $0.95(\sin(\gamma_it+\delta_i)+1)/i$ and latter with probability $1-0.95(\sin(\gamma_it+\delta_i)+1)/i$. States $s_1, \ldots, s_{n+1}$ admit only one action, which moves the agent immediately back to $s_0$. The rewards at states $s_i$, $i\in\{1,\ldots,n\}$, equal $i$, while the rewards at states $s_0$, regardless of the action, and state $s_{n+1}$ equal $0$. For simplicity of the narrative, we assume that both transitions from $s_0$ to some $s_i$ and from $s_i$ to $s_0$ occur together in one time step instead of two.

It is clear that, as $t\to\infty$, the average reward produced by each arm $i$, $i\in\{2,\ldots,n\}$, converges to $0.95$. Thus, an agent that uses classical estimation for learning and planning and pulls the arm that is estimated to bring the highest reward will always eventually begin solely using arm $1$. The average collected reward will thus converge to $1$.

On the other hand, if sets $\{\gamma_i~|~i\geq 2\}$ or $\{\delta_i~|~i\geq 2\}$ are ``sufficiently different'', at every time there will exist an arm producing a reward greater than $1$. Thus, by choosing the bandit arms wisely, an agent may collect an average reward greater than $1$: the primary challenge is in deciding when to choose which arm to pull. In combination with the CCMLE, the notion of uncertainty introduced in Section \ref{unc} and the subsequent control policy introduced in Section \ref{contrp} provide a possible solution. By performing exploratory pulls when the uncertainty of probabilities in a particular arm becomes high, the agent is able to detect when an arm yields a high probability of rewards.

We simulated a $5$-arm bandit during $10000$ time steps, with all $\gamma_i$ chosen uniformly at random in $[0,1/5]$ and $\delta_i$ chosen uniformly at random in $[0,2\pi)$. We use the same horizon length as in Section \ref{sim2}. Bound $\varepsilon$ is $0.25$. To illustrate the power of CCMLE, such a bound is extremely loose---namely, it can be analytically shown that it will be at least $2.5$ times larger than the amount of maximal change on any arm, and at least $6.5$ times larger than the amount of maximal change on the fifth arm. Uncertainty weight $\beta$ is set to $3/(2\sqrt{2})$---a number chosen almost accidentally from previous versions of this experiment, thus without any particular tuning to this scenario. We remark that the computational complexity of the CCMLE method and subsequent planning depends only linearly on the number of bandits, as the optimization problems for computation of a CCMLE and uncertainties are performed separately on each bandit. While the number of variables in the optimization problems grows linearly with elapsed time, this dependence can be removed by adopting a notion of forgetfulness explored in Section \ref{sim2}. We adopt such a notion in the simulation, and limit the memorized history of outcomes for each bandit to $1+1/\varepsilon$, in line with the discussion at the end of Section \ref{ccmle}.

Figure~\ref{mbandits} shows the average collected reward attained by an agent that follows the CCMLE method of estimation and decides which arms to pull based on the sum of expected reward and uncertainty bonus. The planner achieves a reward that is $25\%$ higher than the average reward of an agent that does not use CCMLE or an uncertainty bonus. We emphasize that these results were obtained without any tuning of the weight $\beta$ or bound $\varepsilon$, the latter of which was intentionally poorly chosen.

\begin{figure}[t]
\centering
\includegraphics[width=0.55\textwidth]{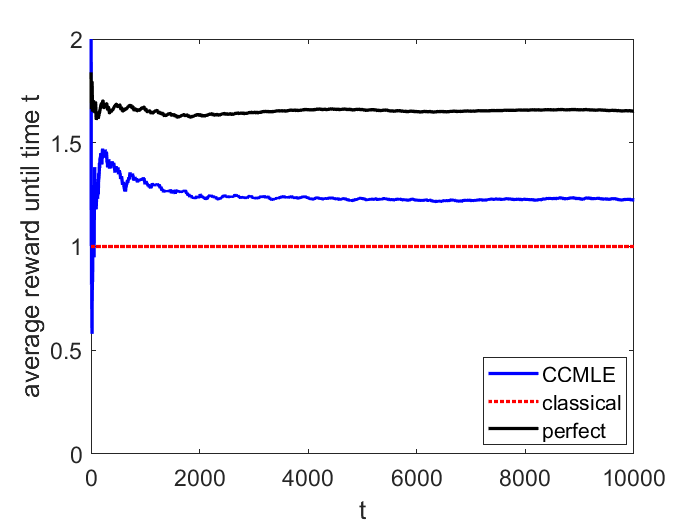}
\caption{Rewards obtained by an agent in a multi-armed bandit scenario. The blue line indicates the average reward obtained by an agent who uses the forgetful CCMLE with an uncertainty-based learning bonus to choose the best among the $5$ bandit arms. The dotted red line indicates the asymptotic average reward obtained by an agent that uses classical estimation unconscious of the change in the transition probabilities. The black line indicates the average reward obtained by an agent who has perfect knowledge of the transition probabilities.}
\label{mbandits}
\end{figure}

\subsection{Wind Flow Estimation}
\label{wfe}

Further building up the complexity of our simulated scenarios, we now provide a more realistic estimation example, simultaneously featuring (i) changes of multiple transition probabilities by different amounts, (ii) complex, time-varying, and partly random rates of change, and (iii) a poorly chosen a~priori bound $\varepsilon$ on the rate of change of transition probabilities, i.e., a bound that is sometimes overly conservative---allowing for changes larger than actual ones---and sometimes incorrect, expecting only changes smaller than the actual ones.

The setting of this simulation is that of an unpowered aerial vehicle without extensive instrumentation, i.e., a pilot balloon.
The movement of such a vehicle---tracked from the ground or using a GPS receiver as its sole instrument---is commonly used to estimate the dynamics in the balloon's area of operation \citep{Hic15}. As the balloon is not equipped with any instruments apart from possibly a GPS receiver, the estimation is based solely on the observed trajectory. In this section, we will develop a CCMLE-based method for estimating the balloon's dynamics in changing wind flow.

We develop our model of the system dynamics based on the work of \citet{AlSetal13}, which devises and discusses a scheme to convert wind direction into transition probabilities on an MDP. Given that the balloon is unpowered, an MDP reduces to a Markov chain, i.e., to an MDP with a single action. Additionally, as the balloon cannot intentionally keep revisiting the same state, in order to enable meaningful estimation, we make the assumption that the dynamics in the area of operation are \textit{uniform}, i.e., do not vary across the state space. While never entirely correct, such an assumption is generally reasonable in the absence of turbulent phenomena \citep{Liuetal12}. We allow for the wind flow and the subsequent dynamics to be time-varying.

We discretize the state space as a grid similar to the one used in Section \ref{sim1}, although we allow such a grid to be infinite or arbitrarily large. At every time step, the balloon moves by one tile in one of four possible directions, denoted by their angle with the positive ray of the $x$-axis: north ($\pi/2$), east ($0$), south ($3\pi/2$), or west ($\pi$). Combined with the assumptions of uniformity and time-varying nature of the wind flow, the presented setting yields a probability transition function $P:\{0,\pi/2,\pi,3\pi/2\}\times \NN_0\to[0,1]$, where $P(o,t)$ signifies the probability at time $t$ that the balloon will move in the direction denoted by $o$.

For simplicity, we consider wind speed to be constant, and only consider changes in its direction $d(t)\in[0,2\pi)$. Adapting the work of \citet{AlSetal13}, we develop the following model for transition probabilities. Given the wind direction $d$, the balloon movement $o\in \{0,\pi/2,\pi,3\pi/2\}$ is a discrete random variable obtained by rounding $O$ to the nearest $\pi/2$ (modulo $2\pi$), where $O$ is a normal random variable with mean $\mathbb{E}[O]=d$ and variance $\sigma^2=1/2$. In other words, $$\Prob(o|d)=\sum_{n=-\infty}^{\infty}\int_{2n\pi+o-\pi/4-d}^{2n\pi+o+\pi/4-d}\frac{2}{\sqrt{\pi}}e^{-\omega^2}d\omega\textrm{.}$$ The transition probability $P(o,t)$ is then naturally given by $P(o,t)=\Prob(o|d(t))$. 
We emphasize that the estimator is neither aware of the wind direction at any time, not aware of the relationship between $d$ and $o$. The estimator is not attempting to establish the wind direction $d(t)$, but solely estimate the transition probabilities $P(\cdot,t)$.

We simulate the wind flow given by \begin{equation*}
\begin{split}
d(t+1) & =d(t)-3\pi/180+X(t)\textrm{,} \\
d(0) & = D_0\textrm{,}
\end{split}
\end{equation*} where $X(t)$ is a random variable whose value is drawn from a uniform distribution on $[-\pi/180,\pi/180]$, and $D_0$ is a random variable whose value is drawn from a uniform distribution on $[0,2\pi]$. In other words, the wind changes direction by $2$ to $4$ degrees at every time step, resulting in continually changing transition probabilities.

To demonstrate the robustness of the CCMLE approach, we chose $\varepsilon=0.03$ as the bound on rate of change in $P$. Such a bound is intentionally incorrect; the actual change in $P$ can be computed from the equations above to be between $0.01$ and $0.04$. Figure~\ref{windrun} compares the maximal error in the transition probabilities of a CCMLE approach with $\varepsilon=0.03$ and a classical approach. It also provides the error in transition probabilities for the CCMLE with a correct bound of $\varepsilon=0.05$, which produces similar results to that of $\varepsilon=0.03$. The figure illustrates the estimation errors for $240$ time steps, allowing the wind direction vectors to make around two full circles.

\begin{figure}[t]
\centering
\includegraphics[width=0.55\textwidth]{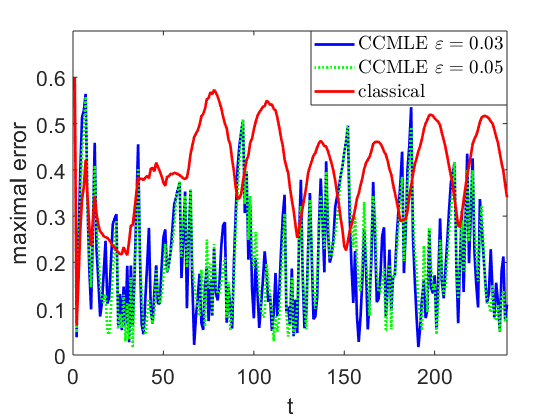}
\caption{Maximal error in the transition probability estimates for the wind flow estimation problem. The red curve indicates the error with classical estimation that assumes time-invariant wind flow. The blue curve indicates the error with the CCMLE method with $\varepsilon=0.03$. The green dotted curve indicates the error with the CCMLE method with $\varepsilon=0.05$.}
\label{windrun}
\end{figure}

The simulated setting with $\varepsilon=0.03$ is not covered by the developed theory of CCMLE. Nonetheless, the CCMLE approach results in significantly better estimation than the classical estimation method. While ``spikes'' in errors are provably unavoidable---when the transitions are non-deterministic, any arbitrarily large sequence of low-probability outcomes will eventually occur---the error from the CCMLE exceeds the error of classical estimation only on a handful of occasions. The average of the maximal CCMLE error over time (both with a slightly incorrect bound of $\varepsilon=0.03$ and a correct bound of $\varepsilon=0.05$) equals $0.21$, while the average of the maximal error in the classical estimation method is nearly double, equaling $0.4$.

\section{Conclusions}
\label{conc}

The work in this paper concentrated on presenting an integrative method for estimation, learning, and planning of an agent operating in an unknown TVMDP. The proposed method is founded on introducing three notions:
\begin{itemize}
\item \textit{change-conscious maximal likelihood estimation (CCMLE)}, which exploits the knowledge on the maximal possible rate of change of transition probabilities to produce time-varying estimates based on the observed outcomes of agent's actions;
\item \textit{uncertainty of an estimate}, which quantifies the lack of knowledge about transition probabilities at a particular time during the system run; and
\item \textit{optimal control policy with an uncertainty-based learning bonus}, which aims to enable the agent to actively learn about transition probabilities in order to increase its long-term attained reward.
\end{itemize}

As shown in Proposition \ref{gener}, when used in a time-invariant MDP, the CCMLE produces the same estimates as the classical method based on the frequency of all previously observed outcomes. On the other hand, as indicated by the theoretical results of Section \ref{estim} and validated on the numerical examples in Section \ref{sim}, in a time-varying setting the CCMLE produces significantly better estimates of transition probabilities than the method based on the frequency of all previously observed outcomes, which implicitly assumes that the transition probabilities are time-invariant. Similarly, the notion of uncertainty introduced in Section \ref{unc} reduces to previous methods for describing the lack of knowledge about transition probabilities in time-invariant MDPs, while providing a novel measure of the lack of knowledge about transition probabilities in the time-varying setting. As proposed in Section \ref{pols}, such a measure can then be used to design a learning bonus for a learning and control policy of an agent operating in a TVMDP, once again generalizing the active learning and control policies previously introduced for time-invariant MDPs. Numerical examples of Section \ref{sim} show that the proposed policies enable an agent to successfully learn the transition probabilities and achieve its control objective, surpassing the outcomes of classical estimation methods. 

The work presented in this paper presents an initial discussion of optimal estimation and learning methods for time-varying stochastic control processes. While theoretical results and numerical examples encourage future exploration of the CCMLE and CCMLE-based learning and planning, these results are by no means exhaustive. Namely, the behavior of the CCMLE in the case of changing transition probabilities has been theoretically explored only for transition probabilities which eventually equal $1$. A natural next step is to obtain theoretical bounds for estimation error and convergence to correct transition probabilities in the case when all transition probabilities are in $[0,1)$. On the side of active learning, Theorem \ref{next} relates the measure of uncertainty introduced in this paper to exploration bonuses used in time-invariant MDPs. However, there are currently no formal guarantees---parallel to the work that uses classical exploration bonuses---on the convergence to optimal control policy when using an uncertainty-based learning bonus, either in a time-varying or time-invariant setting. The role of bonus multiplier $\beta$ remains to be discussed---while for $\beta=0$ the agent does not actively learn, and for $\beta\to\infty$ the agent solely learns and does not have any incentive to maximize its collective reward, it is currently unknown how $\beta$ should be chosen to obtain an optimal policy. Additionally, while the example of Section \ref{mab} shows that one could extend the presented theory of CCMLE to tackle learning and planning with time-varying unknown rewards, we leave the issue of general unknown rewards for future work. Along the same line of thinking, it would be fascinating to combine the recent work of \citet{Ortetal20} with the one proposed in this paper: the planning policy for TVMDPs used by \citet{Ortetal20} results in strong theoretical guarantees on collected rewards, even if the rewards are a priori unknown. However, the estimation method of \citet{Ortetal20} is simple and based on classical estimation. Replacing it with the CCMLE may lead to an increase in the overall collected rewards, possibly adapting the guarantees of \citet{Ortetal20} to a new joint learning and planning method.

Finally, we note that the current paper presents the CCMLE solely in the framework of TVMDPs. The same method, however, may be easily adapted to the framework of general online learning and planning, where the problem is to estimate a time-varying signal from a time series of observations, where the signal is known to change with a rate no greater than some a priori known bound, and use the estimate to plan future actions. A preliminary approach to learning in such a framework has been introduced by \citet{ElKOrn20}, drawing from the current manuscript, but it largely focuses on an application to hypersonic flight vehicles. A similar problem has also been considered by \citet{YuaLam19}; however, instead of the CCMLE approach of attempting to find the most likely time-varying parameter satisfying the constraint on the maximal rate of change, \citet{YuaLam19} attempt to find the time-varying signal that offers the least regret against a comparison sequence that satisfies a similar constraint. The CCMLE problem for general online learning and planning, with all the questions opened in this paper, thus remains largely open.

\acks{This work was partly supported by an Early Stage Innovations grant from NASA's Space Technology Research Grants Program (grant no.\ 80NSSC19K0209), Sandia National Laboratories grant no.\ 2100656, National Science Foundation grants no.\ 1646522 and no.\ 1652113, Defense Advanced Research Projects Agency grant no.\ W911NF-16-1-0001, and Office of Naval Research grant no.\ N00014-20-1-2115. We thank the anonymous reviewers for their valuable input.}

\appendix

\section{Proofs of Theoretical Results}
\label{appx}
\begin{proof}\textbf{of Proposition \ref{gener}}\quad
To simplify notation, assume that $|A|=1$, i.e., $A=\{a\}$. As alluded to in Section \ref{estim}, discrete distributions $\tilde{P}(s,a,\cdot,t)$ for different $s\in S$ never appear together in the constraints of \eqref{bigop2}. Hence, \eqref{bigop2} can be separated into $|S|$ problems
\begin{equation}
\label{bigopne}
\begin{array}{{>{\displaystyle}c}*2{>{\displaystyle}l}}
\min_{\tilde{P}^T(s,a,\cdot,\cdot)} & \quad -\sum_{i=0}^{\#(s,a)}\log \tilde{P}^T(s,a,s_{\nu_{i;s}+1},\nu_{i;s})  \\
\textrm{s.t.} & \quad \tilde{P}^T(s,a,s',t)\geq 0 & \quad \textrm{ for all } s'\in S, t<T\textrm{,} \\ & \quad
\sum_{s'\in S} \tilde{P}^T(s,a,s',t)=1\textrm{,}
& \quad \textrm{ for all } t<T\textrm{,} \\ & \quad |\tilde{P}(s,a,s',t+1)-\tilde{P}(s,a,s',t)|\leq\varepsilon_t & \quad \textrm{ for all } s'\in S\textrm{, } t<T\textrm{,}\end{array}
\end{equation}
where $\nu_{i;s}$ denotes the time at which state $s$ has been visited $i$-th time.

When $\varepsilon_t=0$, \eqref{bigopne} devolves into \begin{equation}
\label{bigop4}
\begin{array}{{>{\displaystyle}c}*2{>{\displaystyle}l}}
\min_{\tilde{P}^T(s,a,\cdot,*))} & \quad -\sum_{i=0}^{\#(s,a)}\log \tilde{P}^T(s,a,s_{\nu_{i;s}+1},*)  \\
\textrm{s.t.} & \quad \tilde{P}^T(s,a,s',*)\geq 0 & \quad \textrm{ for all } s'\in S\textrm{,} \\ & \quad
\sum_{s'\in S} \tilde{P}^T(s,a,s',*)=1\textrm{.}
& \end{array}
\end{equation}
Problem \eqref{bigopne} is a standard maximum likelihood estimation problem for a multinomial distribution, and it can be easily shown \citep{Zel06} that the solution is achieved for $\tilde{P}^T(s,a,s',*)=|\{t\in\{0,\ldots,T-1\}~|~s_t=s, s_{t+1}=s'\}|/\#(s,a)=\#(s,a,s')/\#(s,a)$.
\end{proof}

\begin{proof}\textbf{of Theorem \ref{firbig}}\quad
Since we only care about the case $\#(s,a)\to \infty$, we can assume that $T>N$. We will also assume without loss of generality that $N$ is the least integer such that $\varepsilon_t=0$ for all $t\geq N$. As in the proof of Proposition \ref{gener}, we assume that $A=\{a\}$, and will decouple the relevant optimization problem for different states $s$. To emphasize that transition probabilities $P(s,a,\cdot,t)$ are known to be time-invariant for $t\geq N$, we denote them, including the probability $P(s,a,\cdot,T-1)$, by $P(s,a,\cdot,*_{t\geq N})$, and analogously for their estimates.

For $k\geq l\geq 0$, let $\#^k_l(s,a)=|\{t\in\{l,\ldots,k-1\}~|~s_t=s\}|$, $\#^k_l(s,a,s')=|\{t\in\{l,\ldots,k-1\}~|~s_t=s,s_{t+1}=s'\}|$, and let $\nu_{i;s;l}$ be the $i$-th time at which $s$ has been visited on or after time step $l$. Using our previous notation, $\#(s,a)=\#_0^T(s,a)$ and $\#(s,a,s')=\#_0^T(s,a,s')$. 
The objective function in \eqref{bigopne} equals 
\begin{equation}
\label{eacht}
\begin{aligned}
& -\sum_{i=0}^{\#^N_0(s,a)}\log \tilde{P}^T(s,a,s_{\nu_{i;s;0}+1},\nu_{i;s;0})-\sum_{i=0}^{\#^T_N(s,a)}\log \tilde{P}^T(s,a,s_{\nu_{i;s;N}+1},*_{t\geq N}) \\
& = -\sum_{i=0}^{\#^N_0(s,a)}\log \tilde{P}^T(s,a,s_{\nu_{i;s;0}+1},\nu_{i;s;0})-\sum_{s'\in S}\#_N^T(s,a,s')\log \tilde{P}^T(s,a,s',*_{t\geq N})\textrm{.}
\end{aligned}
\end{equation}
From our derivation, if $\#_N^T(s,a,s')=0$, then we take $\#_N^T(s,a,s')\log \tilde{P}^T(s,a,s',*_{t\geq N})=0$ regardless of the finite or infinite value of $\log\tilde{P}^T(s,a,s',*_{t\geq N})$.

We note that $\#^N_0(s,a)$ is independent of $T$. Hence, if $\#(s,a)\to \infty$, then $\#^T_N(s,a)=\#(s,a)-\#^N_0(s,a)\to \infty$. If we divide the objective function \eqref{eacht} by $\#_N^T(s,a)$, we will obviously not change the solutions of the relevant optimization problem. The solutions of \eqref{bigopne} are thus the solutions of
\begin{equation}
\label{eacht2}
\begin{aligned}
m_1=\min_{\tilde{P}} -\sum_{i=0}^{\#^N_0(s,a)}\frac{\log \tilde{P}^T(s,a,s_{\nu_{i;s;0}+1},\nu_{i;s;0})}{\#_N^T(s,a)}-\sum_{s'\in S}\frac{\#_N^T(s,a,s')}{\#_N^T(s,a)}\log \tilde{P}^T(s,a,s',*_{t\geq N})\textrm{,}
\end{aligned}
\end{equation}with constraints
\begin{equation}
\label{newconst}
\begin{array}{{>{\displaystyle}c}*2{>{\displaystyle}l}}
\tilde{P}^T(s,a,s',t)\geq 0 & \quad \textrm{ for all } s'\in S, t<N\textrm{,} \\
\tilde{P}^T(s,a,s',*_{t\geq N})\geq 0 & \quad \textrm{ for all } s'\in S\textrm{,} \\
\sum_{s'\in S} \tilde{P}^T(s,a,s',t)=1\textrm{,}
& \quad \textrm{ for all } t<N\textrm{,} \\
\sum_{s'\in S} \tilde{P}^T(s,a,s',*_{t\geq N})=1\textrm{,}
& \\
|\tilde{P}(s,a,s',t+1)-\tilde{P}(s,a,s',t)|\leq\varepsilon_t & \quad \textrm{ for all } s'\in S\textrm{, } t<N-1\textrm{,} \\
|\tilde{P}(s,a,s',*_{t\geq N})-\tilde{P}(s,a,s',N-1)|\leq\varepsilon_{N-1} & \quad \textrm{ for all } s'\in S\textrm{.}
\end{array}
\end{equation}
Let, for each $T>N$, $\tilde{P}^T(s,a,\cdot,\cdot)$ be any solution of \eqref{eacht2}. 

Let us take any $s'\in S$. We first claim the following: unless $P(s,a,s',*_{t\geq N})=0$, there exists, with probability $1$, a $\delta>0$ and a large enough $T'$ such that $\tilde{P}^T(s,a,s',*_{t\geq N})\geq \delta$ for all $T\geq T'$. Assume otherwise. Then $\tilde{P}^{T_i}(s,a,s',*_{t\geq N})\to 0$ for some sequence $T_i\to +\infty$. However, as $\#_N^{T_i}(s,a,s')/\#_N^{T_i}(s,a)\to P(s,a,s',*_{t\geq N})\neq 0$ by the law of large numbers, the function value of  $\tilde{P}^{T_i}(s,a,\cdot,\cdot)$ in the objective function of \eqref{eacht2} goes to $+\infty$ as $T_i\to +\infty$. On the other hand, the objective function value of $P(s,a,\cdot,\cdot)$ converges to the finite number $-\sum_{s\in S'}P(s,a,s',*_{t\geq N})\log P(s,a,s',*_{t\geq N})$. Hence, for large enough $T_i$, $\tilde{P}^{T_i}(s,a,\cdot,\cdot)$ could not be an optimal solution of \eqref{eacht2}. Having proved the claim, we observe that, since there are only finitely many $s'\in S$, we can take $\delta>0$ such that the above claim holds for all $s'$ with the same $\delta$.

Let us now compare \eqref{eacht2} with \begin{equation}
\label{rework}
m_2=\min_{\overline{P}} -\sum_{s'\in S}P(s,a,s',*_{t\geq N})\log \overline{P}(s,a,s',*_{t\geq N})\textrm{,}
\end{equation}
where $\overline{P}$ satisfies the constraints from \eqref{newconst}. We assume the convention $0\log 0=0$. By an analogous discussion to that under \eqref{bigop2}, problem \eqref{rework} admits a solution. Let $\overline{P}(s,a,\cdot,\cdot)$ be a solution of \eqref{rework}. 

We observe the following facts:
\begin{enumerate}[(i)]
    \item Value of $m_2$ is less than or equal to the objective function value of $\tilde{P}^T(s,a,\cdot,\cdot)$ in \eqref{rework} for all $T$, by definition of $m_2$,
    \item For every $\varepsilon>0$, there exists $T'$ such that for all $T>T'$, the objective function value of $\tilde{P}^{T'}(s,a,\cdot,\cdot)$ in \eqref{rework} is less than or equal to the objective function value of $\tilde{P}^{T'}(s,a,\cdot,\cdot)$ in \eqref{eacht2} plus $\varepsilon$, i.e., of $m_1+\varepsilon$. The proof for this claim is as follows. The difference between the value in \eqref{rework} and in \eqref{eacht2} equals 
    \begin{equation}
    \label{eacht3}
    \begin{split}
    & \sum_{i=0}^{\#^N_0(s,a)}\frac{\log \tilde{P}^T(s,a,s_{\nu_{i;s;0}+1},\nu_{i;s;0})}{\#_N^T(s,a)} \\ & \quad +\sum_{s'\in S}\left(\frac{\#_N^{T}(s,a,s')}{\#_N^{T}(s,a)}-P(s,a,s',*_{t\geq N})\right)\log \tilde{P}^T(s,a,s',*_{t\geq N})\textrm{.}
    \end{split}
    \end{equation}
    The first sum is nonpositive. If $P(s,a,s',*_{t\geq N})=0$, then $\#_N^{T}(s,a,s')=0$ with probability $1$. Then, as discussed below \eqref{eacht2}, we take $$\left(\frac{\#_N^{T}(s,a,s')}{\#_N^{T}(s,a)}-P(s,a,s',*_{t\geq N})\right)\log \tilde{P}^T(s,a,s',*_{t\geq N})=0\textrm{.}$$
    If $P(s,a,s',*_{t\geq N})\neq 0$, we showed previously that there exists $\delta>0$ such that $\log \tilde{P}^T(s,a,s',*_{t\geq N})\geq \log \delta$ for all large enough $T$. Naturally, $\log \tilde{P}^T(s,a,s',*_{t\geq N})\leq 0$. Expression \eqref{eacht3} thus cannot exceed $$\sum_{s'\in S}\left|\left(\frac{\#_N^{T}(s,a,s')}{\#_N^{T}(s,a)}-P(s,a,s',*_{t\geq N})\right)\right||\log\delta|\textrm{.}$$ By the law of large numbers, with probability $1$ each of the summands above converge to $0$. Hence, for a large enough $T$, the value of \eqref{eacht3} will not exceed $\varepsilon$.
    \item The value of $m_1$ is less than or equal to the objective function value of $\overline{P}(s,a,\cdot,\cdot)$ in \eqref{eacht2} for all $T$, by definition of $m_1$.
    \item For every $\varepsilon>0$, there exists $T'$ such that for all $T>T'$, the objective function value of $\overline{P}(s,a,\cdot,\cdot)$ in \eqref{eacht2} is less than or equal to the objective function value of $\overline{P}(s,a,\cdot,\cdot)$ in \eqref{rework} plus $\varepsilon$, i.e., of $m_2+\varepsilon$. The proof for this claim is as follows. First, note that, since probabilities at time steps prior to $N$ do not come into the objective function of \eqref{rework}, we can without loss of generality ``overwrite'' all the values of $\overline{P}(s,a,s',t)$ with $\overline{P}(s,a,s',*_{t\geq N})$ and thus assume that either $\overline{P}(s,a,s',t)\neq 0$ for all $T<N$ or  $\overline{P}(s,a,s',*)=0$. 
    
    We again subtract the objective function of \eqref{eacht2} from \eqref{rework} and obtain \eqref{eacht3}, just with $\overline{P}$ instead of $\tilde{P}^T$. Now, $\overline{P}(s,a,s',*_{t\geq N})$ does not depend on $T$, so the second sum in \eqref{eacht3} this time trivially converges to $0$ by the law of large numbers. 
    
    Since the values of $\overline{P}(s,a,s',t)$ do not come into the objective function of \eqref{rework} for $t<N$, the following claim holds: for any solution $\overline{P}(s,a,\cdot,\cdot)$, the solution in which we just set $\overline{P}(s,a,s',t)=\overline{P}(s,a,s',*_{t<N})$ and take any values for those probabilities will also be minimal, as long as they stay within the constraints. Since $N$ is assumed to be the least integer such that $\varepsilon_t=0$ for all $t\geq 0$, we can perturb each value of $\overline{P}(s,a,s',*_{t\geq N})$ by some amount smaller than $\varepsilon_{N-1}>0$ so that no value equals $0$. Thus, we can consider $\overline{P}$ such that $\overline{P}(s,a,s',t)=\overline{P}(s,a,s',*_{t<N})\neq 0$ for all $s'\in S$ and $t<N$.
    
    Since we took that $\overline{P}(s,a,s_{\nu_{i;s;0}+1},\nu_{i;s;0})$ does not equal $0$, and it does not depend on $T$, the first sum in \eqref{eacht3} also converges to $0$ by the law of large numbers.
\end{enumerate}

By combining (i)-(iv), we now showed that for every $\varepsilon>0$, there exists $T'$ such that for all $T>T'$, the objective function value of $\tilde{P}^{T'}(s,a,\cdot,\cdot)$ in \eqref{rework} is between $m_2$ and $m_2+2\varepsilon$. 

It can easily be shown that, if $P(s,a,s',*_{t\geq N})=0$, $\overline{P}(s,a,s',*_{t\geq N})=0$ is optimal for \eqref{rework}---decreasing $\overline{P}(s,a,s',*_{t\geq N})=0$ and increasing any other $\overline{P}(s,a,s'',*_{t\geq N})$ will only decrease the function value. Hence, we know that the solution to \eqref{rework} is given by $\overline{P}(s,a,s',*_{t\geq N})=0$ whenever $P(s,a,s',*_{t\geq N})=0$. Values of $\overline{P}(s,a,\cdot,t)$ for $t<N$ can, as discussed, be anything that satisfies the problem constraints. Finally, we note that the solution that maximizes $\sum P(s,a,s',*_{t\geq N})\log\overline{P}(s,a,s',*_{t\geq N})$ where all $P(s,a,s',*_{t\geq N})>0$ is the same as the solution that maximizes $\prod \overline{P}(s,a,s',*_{t\geq N})^{P(s,a,s',*_{t\geq N})}$. By simple KKT analysis, we obtain that the solution to this problem is unique and rather obvious: $\overline{P}(s,a,s',*_{t\geq N})=P(s,a,s',*_{t\geq N})$. Because of the uniqueness of the solution, and given that we know that the objective function value of $\tilde{P}^{T'}(s,a,\cdot,\cdot)$ in \eqref{rework} converges to the minimum, we obtain $\tilde{P}^{T'}(s,a,\cdot,*_{t\geq N})\to \overline{P}(s,a,s',*_{t\geq N})$.
\end{proof}

\begin{proof}\textbf{of Theorem \ref{secbig}}\quad
Suppose first that $(s,a)$ has not been visited before $t=N$. Let $N\leq T_0<T_1<\ldots$ denote the times at which $(s,a)$ is been visited. After decoupling \eqref{bigop2} into $|A||S|$ optimization problems by fixing $s\in S$ and $a\in A$, the choice $\tilde{P}^{T_i+1}(s,a,s',\cdot)=1$ and $\tilde{P}^{T_i+1}(s,a,s'',\cdot)=0$ for all $s''\neq s'$, for all $i\geq 0$, clearly minimizes the objective function $$-\sum_{r=0}^i \log\tilde{P}^{T_i+1}(s,a,s',T_r)$$ for all $i\geq 0$, while satisfying the constraints of \eqref{bigop2}.

We now consider the case when $(s,a)$ has been visited before $t=N$. Let $T_0$ be the last time at which $(s,a)$ is visited before $t=N$, and let $T_0<T_1<\ldots<T_{k-1}<N+1/\varepsilon\leq T_k<T_{k+1}<\ldots$, $k\geq 1$, denote all times at which $(s,a)$ is visited starting at $T_0$. We claim that
\begin{equation}
\label{claimthe}
\tilde{P}^{T_i+1}(s,a,s',T_i)\in \left[\min\left(1,\tilde{P}^{T_{i-1}+1}(s,a,s',T_{i-1})+(T_i-T_{i-1})\varepsilon\right),1\right]
\end{equation}
for all $i\geq 1$. 

Assume that \eqref{claimthe} does not hold. Since $\tilde{P}^{T_i+1}(s,a,s',T_i)\leq 1$, we have 
\begin{equation}
\label{addeq}
\tilde{P}^{T_i+1}(s,a,s',T_i)<\min\left(1,\tilde{P}^{T_{i-1}+1}(s,a,s',T_{i-1})+(T_i-T_{i-1})\varepsilon\right)\textrm{.}
\end{equation}
Now, define an alternative choice of transition probabilities as follows: 
\begin{equation}
\label{newp}
\underline{P}^{T_i+1}(s,a,s^*,T)=
\begin{cases}
\tilde{P}^{T_{i-1}+1}(s,a,s^*,T) \qquad \qquad \textrm{ for all } s^*\in S\textrm{, } T\leq T_{i-1}\textrm{,} \\
\min\left(1,\tilde{P}^{T_{i-1}+1}(s,a,s',T_{i-1})+(T_i-T_{i-1})\varepsilon\right) \textrm{ if } s^*=s' \textrm{ and } T=T_i\textrm{.}
\end{cases}
\end{equation}
Let $d_P^1=\underline{P}^{T_{i}+1}(s,a,s',T_i)-\tilde{P}^{T_{i-1}+1}(s,a,s',T_{i-1})>0$. We define $\underline{P}^{T_{i}+1}(s,a,s^*,T_{i})$ for $s^*\neq s'$ in the following way: if $S=\{s^1,\ldots,s^n\}$, where $s'=s^1$, then recursively define 
\begin{align}
\label{newp21}
& \underline{P}^{T_{i}+1}(s,a,s^r,T_{i}) =\max\left(0,\tilde{P}^{T_{i-1}+1}(s,a,s^r,T_{i-1})-d_P^{r-1}\right)\textrm{,} \\
\label{newp22}
& d^r =\tilde{P}^{T_{i-1}+1}(s,a,s^r,T_{i-1})-\underline{P}^{T_{i}+1}(s,a,s^r,T_{i})\textrm{,} \\
\label{newp23}
& d_P^r=d_P^{r-1}-d^r\textrm{,}
\end{align}
for $r\geq 2$. We also define $d^1=d_P^1$.

We will show that $\underline{P}^{T_{i}+1}(s,a,\cdot,T_{i})\geq 0$ as defined in \eqref{newp}--\eqref{newp23} is a legitimate discrete probability distribution:
\begin{enumerate}
\item[1)] By \eqref{newp} and \eqref{newp21}, $\underline{P}^{T_{i}+1}(s,a,s^r,T_{i})\geq 0$ for all $r$. 
\item[2)] By \eqref{newp} and \eqref{newp22}, \begin{equation*}
\begin{split}
& \sum_{r=1}^n \underline{P}^{T_{i}+1}(s,a,s^r,T_{i})=\sum_{r=1}^n\tilde{P}^{T_{i-1}+1}(s,a,s^r,T_{i-1})+d^1-\sum_{r=2}^n d^r \\
& =1+d^1-\sum_{r=2}^n d^r=1+d_P^n\geq 1\textrm{.}
\end{split}
\end{equation*}
We claim that $d_P^n=0$. If $\tilde{P}^{T_{i-1}+1}(s,a,s^r,T_{i-1})\geq d_P^{r-1}$ for any $r\geq 2$, then $d_P^r=0$ by \eqref{newp21}. Subsequently $d^{r+1}=0$ by \eqref{newp21}--\eqref{newp22} and thus $d_P^{r+1}=0$ by \eqref{newp23}. Continuing onwards, we get $d_P^n=d_P^{n-1}=\cdots=d_P^r=0$. Thus, if $d_P^n>0$, then $\tilde{P}^{T_{i-1}+1}(s,a,s^r,T_{i-1})< d_P^{r-1}$ for all $r\geq 2$. Hence, by \eqref{newp21}, $\underline{P}^{T_{i}+1}(s,a,s^r,T_{i})=0$ for all $r\geq 2$. Then, $$1+d_P^n=\sum_{r=1}\underline{P}^{T_{i}+1}(s,a,s^r,T_{i})=\underline{P}^{T_{i}+1}(s,a,s^1,T_{i})\leq 1\textrm{,}$$ where the inequality holds by \eqref{newp}. Thus, $d_P^n\leq 0$, contradicting the assumption $d_P^n>0$.
\end{enumerate}
Additionally, by combining \eqref{newp22} and \eqref{newp23}, and using \eqref{newp21}, $d^r\geq 0$ and $d^r\leq d_P^{r-1}\leq d_P^{r-2}\leq \cdots\leq d_P^1\leq (T_i-T_{i-1})\varepsilon$ for all $r$. Thus, $$\left|\underline{P}^{T_{i}+1}(s,a,s^r,T_{i})-\tilde{P}^{T_{i-1}+1}(s,a,s^r,T_{i-1})\right|\leq (T_i-T_{i-1})\varepsilon$$ for all $r$. Hence, for all $s^*\in S$, we define $\underline{P}^{T_i+1}(s,a,s^*,T)$ for $T\in\{T_{i-1}+1,\ldots,T_i-1\}$ by 
\begin{equation}
\label{newp3}
\underline{P}^{T_i+1}(s,a,s^*,T)=\frac{T_i-T}{T_i-T_{i-1}}\underline{P}^{T_i+1}(s,a,s^*,T_{i-1})+\frac{T-T_{i-1}}{T_i-T_{i-1}}\underline{P}^{T_i+1}(s,a,s^*,T_i)\textrm{,}
\end{equation}
thus ensuring that $|\underline{P}^{T_i+1}(s,a,s^*,T+1)-\underline{P}^{T_i+1}(s,a,s^*,T)|\leq\varepsilon$ remains satisfied for all $T$. It can also be easily verified that $\underline{P}^{T_i+1}(s,a,s^*,T)\geq 0$ for all $s^*$, and that these values sum up to $1$.

We verified that $\underline{P}^{T_i+1}_{t<T_i+1}$, as defined in \eqref{newp}--\eqref{newp3}, satisfies all the constraints in \eqref{bigop2}. The value of the objective function for $\underline{P}^{T_i+1}_{t<T_i+1}$ is strictly lower than for $\tilde{P}^{T_i+1}_{t<T_i+1}$: the values for $\underline{P}^{T_i+1}_{t<T_{i-1}+1}$ have been chosen to be optimal, and the only other element present in the objective function, $\underline{P}^{T_i+1}(s,a,s',T_i)$, satisfies $\underline{P}^{T_i+1}(s,a,s',T_i)>\tilde{P}^{T_i+1}(s,a,s',T_i)$ by \eqref{addeq} and \eqref{newp}. Thus, we reached a contradiction with $\tilde{P}^{T_i+1}_{t<T_i+1}$ being a CCMLE.

Claim \eqref{claimthe} is thus proved. Now, for each $i\geq 1$ we either have $\tilde{P}^{T_i+1}(s,a,s',T_i)\geq \tilde{P}^{T_{i-1}+1}(s,a,s',T_{i-1})+(T_i-T_{i-1})\varepsilon$ or $\tilde{P}^{T_i+1}(s,a,s',T_i)=1$. Assume $\tilde{P}^{T_i+1}(s,a,s',T_i)<1$ for some $i\geq k$, i.e., for $T_i\geq N+1/\varepsilon$. Then, $\tilde{P}^{T_{i-1}+1}(s,a,s',T_{i-1})<1-(T_i-T_{i-1})\varepsilon$. Continuing onwards, we obtain that $\tilde{P}^{T_0+1}(s,a,s',T_0)<1-(T_i-T_0)\varepsilon$. Since $T_i-T_0\geq 1/\varepsilon$, we obtain $\tilde{P}^{T_0+1}(s,a,s',T_0)<1$, i.e., a contradiction.
\end{proof}

\begin{proof}\textbf{of Lemma \ref{lemunc}}\quad
After decoupling \eqref{bigop2}, the objective function for $(s,a)$ equals $$-\sum_{i=0}^k \log\tilde{P}^T(s,a,s_{T_i+1},T_i)\textrm{.}$$
Assume that there exist two solutions $\tilde{P}_1^T(s,a,s_{T_i+1},T_i)$ and $\tilde{P}_2^T(s,a,s_{T_i+1},T_i)$ yielding the same minimal value of the objective function. Then, by a simple convexity argument \citep{Bec14}, $\lambda \tilde{P}_1^T(s,a,s_{T_i+1},T_i)+(1-\lambda) \tilde{P}_2^T(s,a,s_{T_i+1},T_i)$ all need to yield the same value, for all $\lambda\in[0,1]$. Thus, 
\begin{equation}
\label{takeder}
-\sum_{i=0}^k \log\left(\lambda \tilde{P}_1^T(s,a,s_{T_i+1},T_i)+(1-\lambda) \tilde{P}_2^T(s,a,s_{T_i+1},T_i)\right)
\end{equation}
is a constant function of $\lambda\in[0,1]$. Taking the derivative of \eqref{takeder} with respect to $\lambda$, we obtain $$\sum_{i=0}^k\frac{\tilde{P}_1^T(s,a,s_{T_i+1},T_i)-\tilde{P}_2^T(s,a,s_{T_i+1},T_i)}{\lambda \tilde{P}_1^T(s,a,s_{T_i+1},T_i)+(1-\lambda) \tilde{P}_2^T(s,a,s_{T_i+1},T_i)}=0$$
for all $\lambda\in(0,1)$. Taking the second derivative, we obtain $$\sum_{i=0}^k\frac{(\tilde{P}_1^T(s,a,s_{T_i+1},T_i)-\tilde{P}_2^T(s,a,s_{T_i+1},T_i))^2}{(\lambda \tilde{P}_1^T(s,a,s_{T_i+1},T_i)+(1-\lambda) \tilde{P}_2^T(s,a,s_{T_i+1},T_i))^2}=0\textrm{.}$$
In other words, $\tilde{P}_1^T(s,a,s_{T_i+1},T_i)=\tilde{P}_2^T(s,a,s_{T_i+1},T_i)$ for all $i$.
\end{proof}

\begin{proof}\textbf{of Theorem \ref{next}}\quad Since $\varepsilon_t=0$ for all $t\in\NN_0$, any CCMLE will be time-invariant. Thus, $U^t_{\sigma,\alpha}(s,a)$ is the same for all $t\leq T$. We denote the set of all CCMLEs $\tilde{P}^T(s,a,\cdot,*)$ computed at time $T$, with the agent's previous path $\sigma$ and actions $\alpha$, by $\cP^T_{\sigma,\alpha}(s,a,*)$.

If $\#(s,a)=0$, the claim is obvious, as $\cP^T_{\sigma,\alpha}(s,a,*)$ is the entire probability simplex, so $U^t_{\sigma,\alpha}(s,a)=\sqrt{2}$. Assume now that $\#(s,a)\geq 1$. By the proof of Proposition \ref{gener}, $\cP^T_{\sigma,\alpha}(s,a,*)\subseteq\RR^{|S|}$ contains a single element given by components $\tilde{P}^T(s,a,s',*)=\#(s,a,s')/\#(s,a)$, where $\#(s,a,s')$ denotes the number of times that transition $(s,a,s')$ has occurred among the first $T-1$ transitions. Thus, the diameter of $\cP^T_{\sigma,\alpha}(s,a,*)$ is $0$. 

For each $s''\in S$, again by the proof of Proposition \ref{gener}, set $\cP^{T+1}_{\overline{\sigma}_{s''},\overline{\alpha}}(s,a,*)$ also contains a single element given by $\tilde{P}^{T+1}(s,a,s'',*)=(\#(s,a,s'')+1)/(\#(s,a)+1)$ and $\tilde{P}^{T+1}(s,a,s',*)=\#(s,a,s')/(\#(s,a)+1)$ for all $s'\neq s''$.

Thus, the maximum distance between points in polytopes $\cP^{T}_{\sigma,\alpha}(s,a,*)$ and $\cP^{T+1}_{\overline{\sigma}_{s''},\overline{\alpha}}(s,a,*)$ equals 
\begin{equation}
\label{valdis} 
\begin{split}
& \sqrt{\left(\frac{\#(s,a,s'')+1}{\#(s,a)+1}-\frac{\#(s,a,s'')}{\#(s,a)}\right)^2+\sum_{s'\neq s''}\left(\frac{\#(s,a,s')}{\#(s,a)+1}-\frac{\#(s,a,s')}{\#(s,a)}\right)^2} =\\
& \frac{1}{\#(s,a)(\#(s,a)+1)}\sqrt{\sum_{s'\neq s''}\#(s,a,s')^2+(\#(s,a)-\#(s,a,s''))^2}\textrm{.}
\end{split}
\end{equation}
On one hand, by the power mean inequality \citep{Bul03}, the value of \eqref{valdis} is greater than or equal to 
\begin{equation}
\label{geqdis}
\begin{split}
& \frac{1}{\#(s,a)(\#(s,a)+1)}\sqrt{\frac{\left(\sum_{s'\neq s''}\#(s,a,s')\right)^2}{|S|-1}+(\#(s,a)-\#(s,a,s''))^2} \\
& = \frac{1}{\#(s,a)(\#(s,a)+1)}\sqrt{\frac{\left(\#(s,a)-\#(s,a,s'')\right)^2}{|S|-1}+(\#(s,a)-\#(s,a,s''))^2} \\ & =\frac{\left(\#(s,a)-\#(s,a,s'')\right)\sqrt{|S|}}{\#(s,a)(\#(s,a)+1)\sqrt{|S|-1}}\textrm{.}
\end{split}
\end{equation}

We are interested in determining the maximal value of \eqref{valdis} over different $s''\in S$. There exists $s''$ such that $\#(s,a,s'')\leq \#(s,a)/|S|$. By plugging in this $s''$ into \eqref{geqdis}, we obtain that the value of \eqref{valdis} is greater than or equal to $\sqrt{1-1/|S|}/(1+\#(s,a))$.

On the other hand, the value of \eqref{valdis} is trivially less than or equal to $$\frac{1}{\#(s,a)(\#(s,a)+1)}\sqrt{\left(\sum_{s'\in S} \#(s,a,s')\right)^2+\#(s,a)^2}=\frac{\sqrt{2}}{\#(s,a)+1}$$ for any $s''$. By \eqref{defunc}, we obtain the desired claim.
\end{proof}

\bibliography{refs}

\end{document}